%% file: neurips_alt_2026.tex
\documentclass{article}

\PassOptionsToPackage{numbers, compress}{natbib}
\usepackage[preprint]{neurips_2026}

\usepackage[utf8]{inputenc}
\usepackage[T1]{fontenc}
\usepackage{hyperref}
\usepackage{url}
\usepackage{booktabs}
\usepackage{graphicx}
\usepackage{amsfonts}
\usepackage{amsmath}
\usepackage{amssymb}
\usepackage{nicefrac}
\usepackage{microtype}
\usepackage{xcolor}
\usepackage{subcaption}
\usepackage{caption}
\usepackage{wrapfig}
\usepackage{tikz}
\usepackage{amsthm}
\newtheorem{lemma}{Lemma}
\newtheorem{theorem}{Theorem}
\usetikzlibrary{positioning, arrows.meta, calc}

\title{Neural Phase Correlation}

\author{%
  Cole Reynolds \\
  Weyl Labs \\ 
  \texttt{cole@weyllabs.com}
}

\begin{document}

\maketitle

\begin{abstract}

Correspondence is fundamentally relational: it seeks the unknown transformation between two observations of a common scene, not the content of either. Yet the dominant learning-based methods do not represent the transformation as a first-class object in the architecture. They encode each image independently and let a learned similarity function or a deep decoder discover the mapping implicitly. Phase correlation is the canonical exception, measuring the inter-image relationship directly in the Fourier domain, but the rigidity of its fixed basis confines it to global translation.

We introduce a learned generalization of phase correlation that lifts this restriction by learning the basis on which the transformation decomposes. The same algebraic primitive extends to dense non-rigid deformations and to unitary dynamics. On the ACDC cardiac-MRI benchmark the framework matches or exceeds prior published baselines on both registration directions. On CAMUS echocardiography it matches state-of-the-art without auxiliary scoring or adaptive-smoothness mechanisms. Applied to time-evolved wavefunction pairs of the 1-D quantum harmonic oscillator, the same framework recovers the Hermite-function eigenstates and the quantized energy levels of the unknown Hamiltonian from observation pairs alone.
\end{abstract}

\section{Introduction}
\label{sec:intro}

Dense correspondence is a foundational task across imaging domains, establishing a per-pixel mapping between two views of related content. Methods for this problem fall broadly into two architectural traditions. Descriptor-based methods, from classical SIFT\citep{lowe2004sift} to learned local descriptors\citep{revaud2019r2d2}, learned and dense feature matchers\citep{sarlin2020superglue, edstedt2024roma}, and cost-volume networks for optical flow\citep{dosovitskiy2015flownet, sun2018pwcnet, teed2020raft, sun2021loftr}, all encode each image independently and infer correspondence from a similarity computation between feature vectors. Concatenation-based registration networks\citep{balakrishnan2019voxelmorph, chen2022transmorph}, dominant in medical imaging, take the channel-stacked image pair as input to a learned encoder and regress the displacement field directly, never constructing an explicit similarity. Phase correlation \citep{kuglin1975phasecorr} is the structural exception: it measures the inter-image transformation directly in the frequency domain, never constructing a descriptor and never relying on a learned similarity. However, its applicability is bounded by the rigidity of the prescribed Fourier basis: each cross-power-spectrum peak encodes a single shift parameter, confining the method to global translation.

Phase correlation's rigidity is well known; less remarked is what its rigidity buys. When phase correlation fails, the diagnosis is principled: the shift theorem assumed a transformation it cannot represent. This kind of structural transparency, in which failure modes are traceable to specific violated assumptions, has no analog in the descriptor- and concatenation-based architectures that have come to dominate dense correspondence. When the primitive of these methods fails, the response is typically additive, not corrective. While equivariant and geometric deep learning \citep{cohen2016gcnn, worrall2017harmonic, weiler2019e2, cohen2020general} have advanced the structural agenda in adjacent settings, dense correspondence has gone the other way: the architectural commitment to independent encoding has been amplified across two decades of work, yet the alternative of letting the architecture encode part of the transformation directly has remained largely unexplored.

\begin{figure}[t]
\centering
\begin{subfigure}[t]{0.46\linewidth}
\centering
\includegraphics[width=\linewidth]{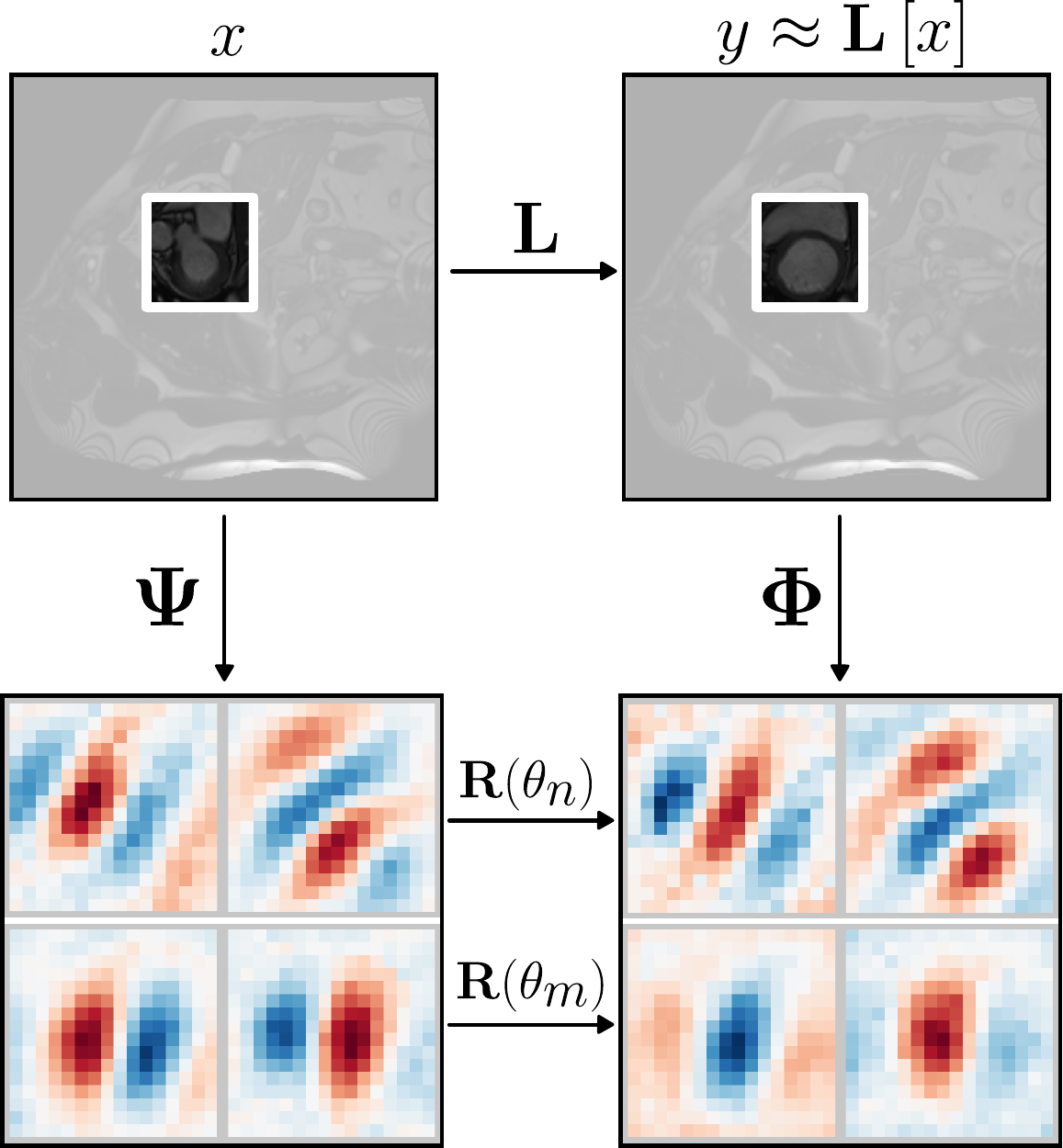}
\caption{Cardiac MRI registration.}
\label{fig:overview_acdc}
\end{subfigure}\hfill
\begin{subfigure}[t]{0.46\linewidth}
\centering
\includegraphics[width=\linewidth]{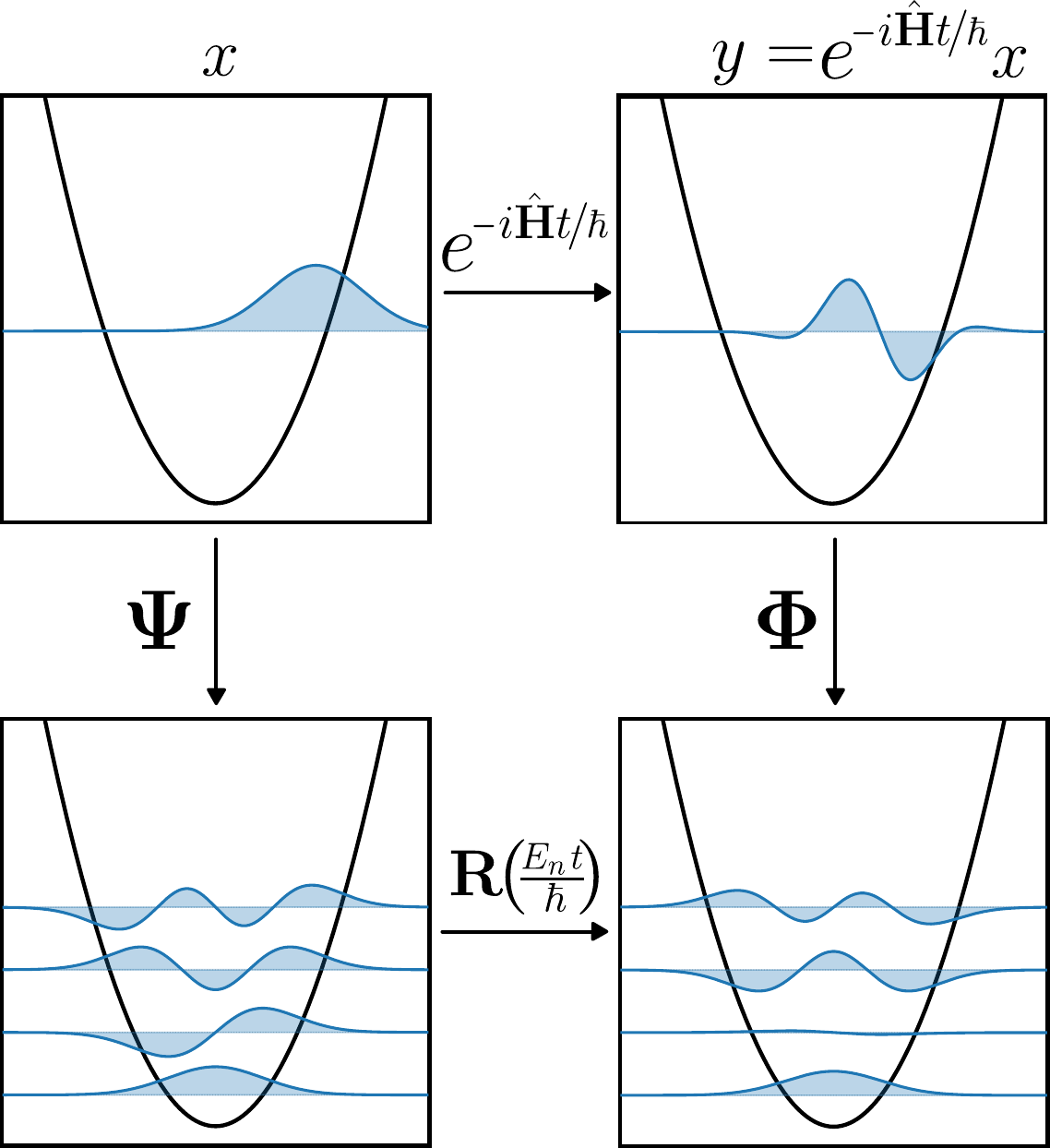}
\caption{1-D quantum harmonic oscillator.}
\label{fig:overview_qho}
\end{subfigure}
\caption{For filter-pair index \(n\), the projection of the moving patch \(\boldsymbol{\psi}_{n}^T x\) and the projection of the fixed patch \(\boldsymbol{\phi}_{n}^T y\) lie in a 2-D subspace on which the local transformation acts as a planar rotation \(\mathbf{R}(\theta_n)\) (Eq.~\eqref{eq:action}, Eq.~\eqref{eq:interaction}). \textbf{Left:} cardiac MRI registration. \textbf{Right:} time-evolution of the 1-D quantum harmonic oscillator, with \(\mathbf{L}=e^{-i\hat{H}t/\hbar}\) and \(\theta_n=E_n t/\hbar\).}
\label{fig:overview}
\end{figure}

The natural generalization, then, is to keep the relational primitive but learn the basis on which it operates. We introduce a framework that recovers the per-subspace rotations of the operator relating two observations, replacing the fixed Fourier basis of classical phase correlation with two learned bases, \(\mathbf{\Psi}\) and \(\mathbf{\Phi}\) (Figure~\ref{fig:overview}). Each pair of filters learns to span a 2-D subspace on which the local operator acts as a planar rotation \(\mathbf{R}(\theta)\), the direct generalization of the per-frequency translation phase that classical phase correlation extracts.

The framework's central mechanism is a closed-form residual \(r_k^2\) (Eq.~\eqref{eq:residual}) that tests, per-location and per-subspace, whether the assumed rotation structure actually holds. The residual gates the architecture in two ways: at inference, it masks subspaces where the structure breaks down, preventing them from corrupting the displacement estimate; during training, top-K selection shapes the filters toward specialists for specific transformations rather than generalists for every transformation (Figure~\ref{fig:residual_dist}). Because the residual is derived from the algebraic model itself, it has no analog in descriptor-based pipelines, and its magnitude reveals where the structural assumption holds and where it fails.

The encoder is task-agnostic: filter pairs \(\mathbf{\Psi}, \mathbf{\Phi}\) and per-pair rotation angles \(\{\theta_n\}\) encode the local transformation regardless of what the transformation represents. The decoder that consumes this encoding is task-specific. For diffeomorphic image registration, the angles drive a dense displacement field \(\varphi(\mathbf{r})\) such that \(y(\mathbf{r})\approx x\left(\mathbf{r}+\varphi(\mathbf{r})\right)\), turning the phase-correlation primitive into a fully differentiable, dense-registration module. For unitary time-evolution \(e^{-i\hat{H}t/\hbar}\) of quantum states, the angle \(\theta_n\) is the eigenenergy phase \(E_n \tau/\hbar\) for some elapsed time \(\tau\), and the filters themselves reconstruct the wavefunction at \(\tau\) in closed form (Section~\ref{sec:qho}).

\paragraph{Contributions.}
\begin{itemize}
  \item A learned generalization of phase correlation that replaces the fixed Fourier basis with paired filter sets (\(\mathbf{\Psi}\), \(\mathbf{\Phi}\)) whose 2-D spans align with invariant subspaces of the local orthogonal operator, extending the phase-correlation primitive from global translation to dense, spatially-varying correspondence under non-rigid deformation.
  \item A closed-form, per-location residual \(r_k^2\) (Eq.~\eqref{eq:residual}, with necessity and almost-sure sufficiency proved in Appendix~\ref{app:residual}) that tests whether each filter-pair span is invariant under the local transformation. The residual masks non-invariant subspaces at inference, shapes the filter set toward consistent specialists during training via per-pair top-K selection (Appendix~\ref{sec:scaling}, Figure~\ref{fig:residual_dist}), and serves as a built-in diagnostic revealing where the structural assumption holds and where it breaks down.
  \item A non-autonomous ODE on the displacement field that re-evaluates the model against the progressively warped moving image at each time step. Coupled to the residual mask, the active subspaces shift across integration time, recovering deformations that a single forward pass cannot.
  \item State-of-the-art on ACDC cardiac MRI (highest published Dice on the harder ED\(\rightarrow\)ES direction; highest across-direction average under matched-loss training) and a match to the strongest CAMUS baseline without its auxiliary scoring or adaptive-smoothness machinery. Deformation breadth: validation on 2-D brain MRI slices and sub-meter synthetic-aperture-radar imagery (Appendix~\ref{app:breadth}), where it outperforms classical phase correlation.
  \item Applied to time-evolved wavefunction pairs of the 1-D quantum harmonic oscillator with randomized, unobserved time intervals \(\tau\), the same primitive recovers the Hermite-function eigenstates and the quantized energy levels of the unknown Hamiltonian from observation pairs alone (Section~\ref{sec:qho}). 
\end{itemize}

More broadly, this work generalizes phase correlation into a tool for recovering orthogonal operators from observation pairs.

\section{Related Work}
\label{sec:related}

\subsection{Medical image registration}
\label{sec:related-medical}

Learned medical registration has converged on a similarity-driven paradigm: from DIRNet~\citep{devos2017dirnet} through VoxelMorph \citep{balakrishnan2019voxelmorph} and its successors \citep{chen2022transmorph, zhang2024adacs, jia2022lkunet, guo2024mambamorph, jia2024lessnet}, two images are stacked as input channels of a learned encoder that regresses a displacement field, inheriting the similarity-driven optimization of classical iterative methods.

\subsection{Spectral and operator learning for dynamical systems}
\label{sec:related-spectral}
Outside image correspondence, several lines learn structural representations of dynamical operators from observation data, each committing to a particular mediating structure: a learned linear latent in the Koopman-net approaches \citep{lusch2018deep, takeishi2017koopman}; a variational wavefunction trained against a known Hamiltonian in neural-network quantum states \citep{carleo2017solving}; a PDE residual at collocation points in physics-informed neural networks \citep{jin2022pinn}; and a parameterized Hamiltonian ansatz \citep{kao2024hamiltonianlearning}. Our framework instead consumes observation pairs directly, learning only the per-pair 2-D invariant-subspace decomposition from which the spectrum follows.

\subsection{Learned Transformation Representations}
Our encoder inherits its bilinear construction from a line of work originating with \citet{memisevic2007unsupervised}, whose three-way multiplicative connections represented inter-image transformations directly but at cubic parameter scaling. The three-way tensor was subsequently factored into per-image filter banks combined multiplicatively at the response level~\citep{memisevic2010factored}, introducing the parallel-bank construction we adopt. They demonstrated empirically that training on global shifts produces Fourier filter pairs in approximate quadrature, with the frequency-domain interpretation made explicit. The factored interaction was reinterpreted as detection of rotation angles within shared invariant eigen-subspaces of the transformation operator~\citep{memisevic2012multiview, memisevic2013relate}. This line demonstrated its principles for transformation detection and Eulerian-styled reconstruction on small image pairs. We extend the bilinear primitive to dense, spatially-varying correspondence at full resolution and add a closed-form per-subspace residual that tests invariance directly.

\section{Method}
\label{sec:method}
\label{sec:pc-reduction}


\begin{figure}[!ht]
\centering
\includegraphics[width=\linewidth]{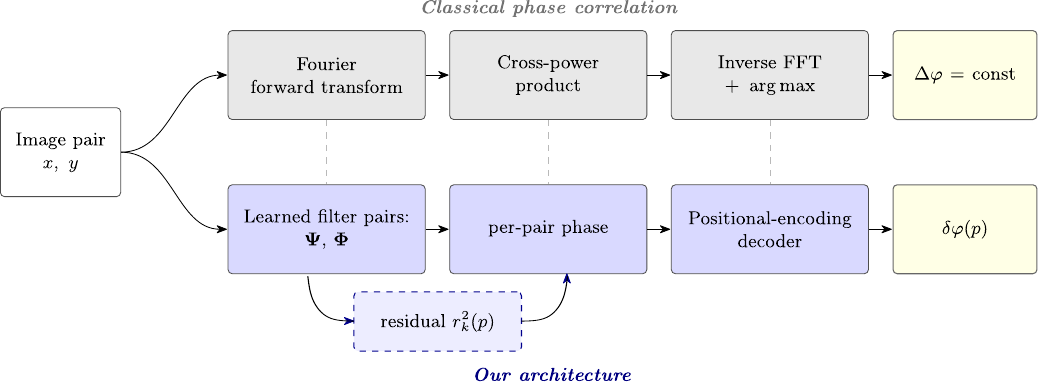}
\caption{The proposed architecture generalizes the three stages of classical phase correlation. The per-pair invariance residual $r_k^2(p)$ (dashed) is a parallel branch with no classical analog: it is computed from the filter-pair responses and yields a top-$K$ mask that selects which filter pairs contribute per-pair phase to the decoder.}
\label{fig:pipeline}
\end{figure}

\begin{table}[!ht]
\centering
\small
\setlength{\tabcolsep}{4pt}
\caption{Stage-by-stage correspondence between classical phase correlation and the proposed architecture.}
\label{tab:pc_correspondence}
\begin{tabular}{lll}
\toprule
Stage & Classical phase correlation & This architecture \\
\midrule
Forward transform    & Fourier + conjugate (fixed)                          & Filter pairs \(\psi_k\), \(\phi_k\) (learned) \\
Per-frequency output & \(\mathrm{Re}/\mathrm{Im}\) of \(\hat x, \overline{\hat y}\) & Pair responses \(\psi_k^T x_k\), \(\phi_k^T y_k\) \\
Cross-power product  & \(\hat x \, \overline{\hat y}\)                       & Bilinear interaction of pair responses \(z_k\) \\
Phase extracted      & Translation angle \(\omega_k \cdot \Delta\varphi\)          & \((\cos\theta_k,\, \sin\theta_k,\, |z_k|)\) \\
Transformation model & Global pure translation                              & Spatially-varying orthogonal \(L_p\) \\
Inverse/integration  & Inverse FFT \(\Rightarrow\) 2-D plane with delta spike & PE residual branch (learned) \\
\bottomrule
\end{tabular}
\end{table}

For clarity, we describe the framework in image-registration language throughout this section and the registration experiments; Section~\ref{sec:qho} re-applies the same algebraic structure with the operator no longer an image warp, with notation adapted to wavefunctions and time evolution. Readers unfamiliar with classical phase correlation may find the brief overview in Appendix~\ref{app:phase-correlation} useful before reading this section.

Image correspondence asks "where does this pixel go?". In the discrete limit it is a permutation \(P\), and permutation matrices are orthogonal (\(P^T P = I\)); the patch-level model \(L_p \in O(n)\) that follows is the natural continuous relaxation.

Let \(x(p), y(p) \in \mathbb{R}^{2C}\) denote the vectorized sub-image patches of the moving and fixed images centered at location \(p\). We model the patch transformation at each location as orthogonal, an approximation that holds locally for small displacements in the registration setting,
\begin{equation}
y(p) = L_p\, x(p), \quad L_p = \mathbf{U} \left[\begin{smallmatrix} R(\theta_1) & & \\ & \ddots & \\ & & R(\theta_C) \end{smallmatrix}\right] \mathbf{U}^{T}, \quad \mathbf{U} = [\hat{u}_1, \ldots, \hat{u}_{2C}].
\label{eq:action}
\end{equation}
Here \(\mathbf{U}\) is the (unknown, location-dependent) eigenbasis of \(L_p\), and the angles \(\theta_c\) likewise depend on \(p\) (notation suppressed for brevity); each \(R(\theta_c)\in SO(2)\) is a planar rotation block on the 2-D subspace spanned by the \(c\)-th column pair of \(\mathbf{U}\). The model interacts with the image pair only through the learned filter banks \(\mathbf{\Psi}\) and \(\mathbf{\Phi}\), which are themselves location-independent. The encoder learns column-pair subspaces \(\boldsymbol{\psi}_k\) and \(\boldsymbol{\phi}_k\) whose spans align with the eigen-subspaces of \(L_p\) at as many locations as possible (full decomposition in Appendix~\ref{app:architecture}); the residual \(r_k^2(p)\) tests whether this alignment holds and gates the architecture accordingly. Classical phase correlation corresponds to the special case where \(\mathbf{U}\) is the spatial Fourier basis: the column pairs are the (\(\cos\), \(\sin\)) components of each frequency, and the bilinear interaction reduces to the cross-power spectrum.

\paragraph{Per-location, per-subspace residual.}
Let \(\boldsymbol{\psi}_{k}, \boldsymbol{\phi}_{k} \in \mathbb{R}^{n \times 2}\) denote the \(k\)-th moving- and fixed-image filter pairs of \(\mathbf{\Psi}\) and \(\mathbf{\Phi}\), each a 2-column matrix in patch space. Because the pairs are learned and the bank is finite, the spans of \(\boldsymbol{\psi}_{k}\) and \(\boldsymbol{\phi}_{k}\) cannot align with an invariant subspace of \(L_p\) at every location. Since orthogonal operators preserve inner products, and \(\boldsymbol{\phi}_{k}^T y \approx \boldsymbol{\phi}_{k}^T L_p\, x\) under the orthogonal model, 
a mismatch between \(\boldsymbol{\psi}_{k}^T x\) and \(\boldsymbol{\phi}_{k}^T y\) is direct 
evidence of mis-alignment. The residual
\begin{equation}
r_k^2(p) \;=\; \bigl(\|\boldsymbol{\phi}_{k}^T y(p)\| - \|\boldsymbol{\psi}_{k}^T x(p)\|\bigr)^2, 
\qquad 
\boldsymbol{\psi}_{k}^T x(p) \;\in\; \mathbb{R}^2,
\qquad
\boldsymbol{\phi}_{k}^T y(p) \;\in\; \mathbb{R}^2,
\label{eq:residual}
\end{equation}
formalizes this test. When \(\boldsymbol{\psi}_{k}\) and \(\boldsymbol{\phi}_{k}\) span a common subspace, orthogonality of \(L_p\) guarantees \(r_k^2(p)=0\) if and only if that span is invariant under \(L_p\) (Lemma~\ref{lem:necessity} and Theorem~\ref{thm:sufficiency} in Appendix~\ref{app:residual}, with sufficiency holding almost surely under mild distributional assumptions on the moving image). The geometrically-normalized form \(\bar{r}_k^2(p)=r_k^2(p)/(\|\boldsymbol{\phi}_k^Ty\|\|\boldsymbol{\psi}_k^Tx\|) \), used as the masking criterion throughout, inherits the same vanishing conditions.

\paragraph{Residual as a masking mechanism.}
A differentiable top-\(K\) selection \citep{xie2020differentiable} retains the \(K\) filter pairs with the smallest residuals at each location, zeroing the remaining \((C-K)\) before they propagate deeper into the network and changing the supervision incentive: the filters are allowed to be "greedy". A pair sharply aligned on some inputs and misaligned on others is rewarded for the former and dropped for the latter. As a result, the bank converges toward specialists rather than generalists, with per-pair residuals at active locations an order of magnitude below the unmasked baseline
(Figure~\ref{fig:residual_dist}).

\begin{figure}[!ht]
\centering
\includegraphics[width=\linewidth]{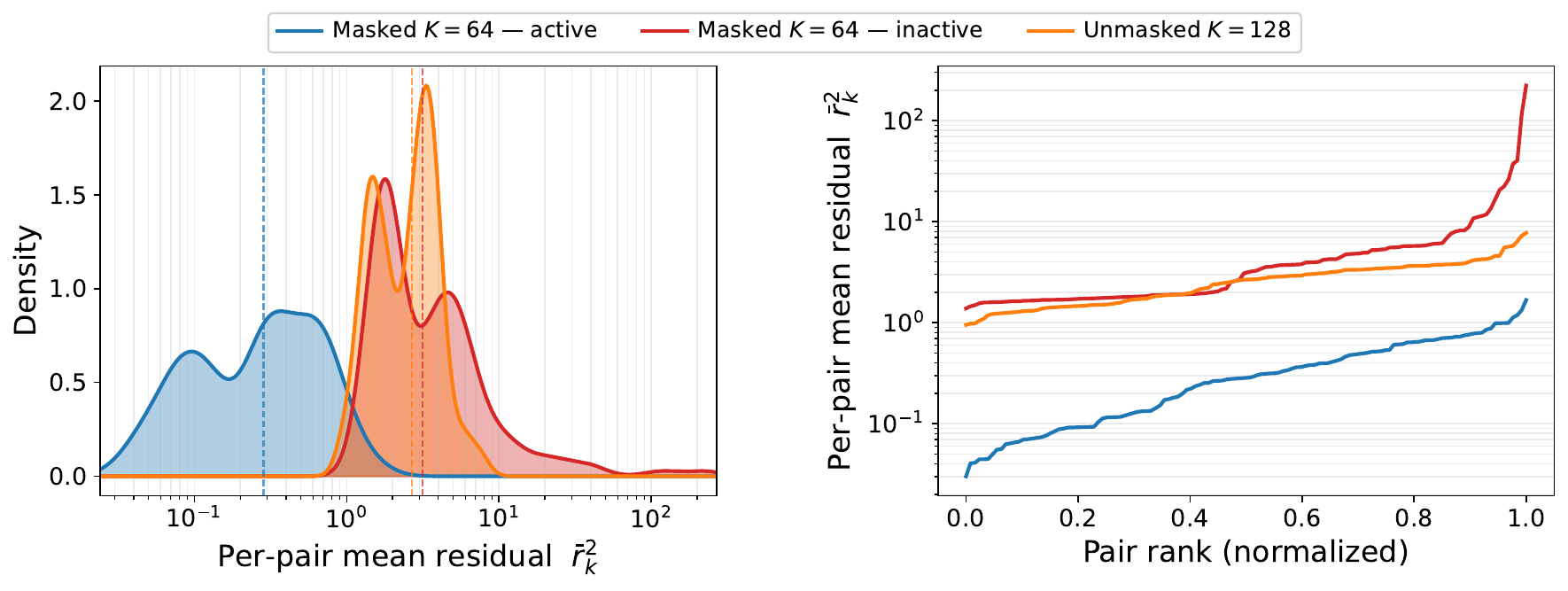}
\caption{Per-pair mean residual \(\bar r_k^2\) on held-out ACDC test pairs (\(C = 128\); masked: \(K = 64\), unmasked: \(K = C\)). \textbf{Left:} Histograms; \textbf{Right:} Sorted per-pair means. The masked distribution is split into \emph{active} pairs (pair survived top-\(K\)) and \emph{inactive} pairs (pair dropped). The histograms make the separation directly visible: at active pairs, masked residuals fall an order of magnitude below the unmasked baseline (median \(0.25\) vs \(2.67\)); at inactive pairs, the same pairs would have matched or exceeded unmasked (median \(\sim 2.55\)). Full interpretation: Appendix~\ref{app:scaling}.}
\label{fig:residual_dist}
\end{figure}

\paragraph{ODE instantiation.}
Each forward pass of the model \(M\) produces a displacement update \(\delta\varphi\). Large deformations are recovered by iterating a non-autonomous ODE on the deformation field:
\begin{equation}
\frac{d\varphi}{dt} \;=\; M\bigl(x \circ \varphi(t),\; y\bigr), \qquad \varphi(0) = \mathrm{id}.
\label{eq:ode}
\end{equation}
The model is re-evaluated at every step against the progressively-refined warped moving image \(x \circ \varphi(t)\), distinct from stationary velocity-field formulations \citep{arsigny2006log} that assume a fixed velocity over the integration interval. Re-evaluating at every step preserves the greediness of the per-pair top-\(K\) mask in time as well as in space: at each \(t\) the active filter pairs are determined by the current warped state, not the initial state. The active subspaces shift with integration time as a result, with low-frequency pairs typically dominating early steps to capture coarse displacement and higher-frequency pairs taking over for fine-scale refinement (Figure~\ref{fig:residual_per_pair_camus_staged}). 

\section{Experiments}
\label{sec:experiments}

\subsection{Cardiac cine MRI (ACDC)}
\label{sec:acdc}

Under our preferred Normalized Cross Correlation (NCC) training recipe, our masked-residual architecture reports 0.863 average Dice on ACDC \citep{bernard2018acdc} and 0.843 on the harder ED\(\rightarrow\)ES direction, exceeding the strongest published baselines (DWCP at 0.837 / 0.813 under MSE). To isolate the architectural contribution from the loss change, we retrain VoxelMorph under our NCC recipe; our architecture exceeds matched-loss VoxelMorph + NCC, and VoxelMorph itself gains a comparable +0.029 from MSE to NCC, so the loss change does not asymmetrically favor our architecture (full loss-control breakdown in Appendix~\ref{app:acdc-protocol}). Under the exact \citet{jian2025disentangling} MSE protocol we exceed the published baselines on the harder ED\(\rightarrow\)ES direction (0.819; Table~\ref{tab:acdc}). Reported numbers are per-class mean \(\pm\) 1-\(\sigma\) sample standard deviation of per-pair test scores from seed 42; cross-seed stability across seeds 42, 43, 44 is reported separately in Appendix~\ref{app:acdc-protocol}.

\begin{table}[h]
\centering
\small
\caption{ACDC bidirectional registration, our 2-D protocol with NCC similarity (vs.\ MSE in
Table~\ref{tab:acdc}); training recipe and metric definitions in
Appendix~\ref{app:acdc-protocol}. The VoxelMorph + NCC row uses our identical recipe; only the
architecture differs.}
\label{tab:acdc_ncc}
\setlength{\tabcolsep}{3pt}
\begin{tabular}{l c c c c c}
\toprule
Method                    & ED\(\rightarrow\)ES Dice    & ES\(\rightarrow\)ED Dice    & Avg                  & SDlogJ            & NDV \(/10\)k       \\
\midrule
Initial (no registration) & \(0.602 \pm 0.148\)         & \(0.602 \pm 0.148\)         & \(0.602\)            & ---               & ---                \\
\midrule
VoxelMorph + NCC          & \(0.817 \pm 0.103\)         & \(0.867 \pm 0.083\)         & \(0.842\)            & \(8.55 \pm 1.35\) & \(83.1 \pm 152.3\) \\
Ours (masked)             & \(\mathbf{0.843 \pm 0.104}\)& \(\mathbf{0.883 \pm 0.080}\)& \(\mathbf{0.863}\)   & \(7.28 \pm 1.79\) & \(76.4 \pm 158.4\) \\
\bottomrule
\end{tabular}
\end{table}

\begin{table}[h]
\centering
\small
\caption{ACDC bidirectional registration, \citet{jian2025disentangling} 2-D MSE protocol. Baselines as reported in their Table~5.}
\label{tab:acdc}
\setlength{\tabcolsep}{6pt}
\begin{tabular}{l c c c}
\toprule
Method                                          & ED\(\rightarrow\)ES Dice    & ES\(\rightarrow\)ED Dice    & Avg            \\
\midrule
Initial (no registration)                       & \(0.602 \pm 0.148\)         & \(0.602 \pm 0.148\)         & \(0.602\)      \\
\midrule
VoxelMorph \citep{balakrishnan2019voxelmorph}   & \(0.788 \pm 0.108\)         & \(0.839 \pm 0.083\)         & \(0.814\)      \\
TransMorph \citep{chen2022transmorph}           & \(0.804 \pm 0.097\)         & \(0.855 \pm 0.074\)         & \(0.830\)      \\
LKU-Net    \citep{jia2022lkunet}                & \(0.777 \pm 0.107\)         & \(0.822 \pm 0.087\)         & \(0.800\)      \\
Mam-VXM    \citep{jian2025disentangling}        & \(0.800 \pm 0.099\)         & \(0.853 \pm 0.076\)         & \(0.827\)      \\
Mam-TM     \citep{jian2025disentangling}        & \(0.804 \pm 0.099\)         & \(0.854 \pm 0.077\)         & \(0.829\)      \\
Dual       \citep{jian2025disentangling}        & \(0.781 \pm 0.110\)         & \(0.831 \pm 0.088\)         & \(0.806\)      \\
DWP        \citep{jian2025disentangling}        & \(0.803 \pm 0.109\)         & \(0.854 \pm 0.085\)         & \(0.829\)      \\
DWCP       \citep{jian2025disentangling}        & \(0.813 \pm 0.103\)         & \(\mathbf{0.861 \pm 0.079}\)& \(\mathbf{0.837}\) \\
DWCPI      \citep{jian2025disentangling}        & \(0.812 \pm 0.104\)         & \(0.858 \pm 0.084\)         & \(0.835\)      \\
\midrule
Ours (masked)                                   & \(\mathbf{0.819 \pm 0.112}\)& \(0.836 \pm 0.103\)         & \(0.828\)      \\
\bottomrule
\end{tabular}
\end{table}

The pairing has structural basis: NCC and our encoder share the same intensity-scale invariance, while MSE is sensitive to intensity bias and scaling that our normalization discards. Each method is reported at its structurally aligned loss (DWCP at 0.813 under MSE, ours at 0.843 under NCC), and the +0.030 Dice gap reflects the architectural payoff rather than a loss-pairing artifact. Representative registrations and residual diagnostics appear in Figures~\ref{fig:registration_examples_acdc} and~\ref{fig:residual_per_pair_acdc}.

\input{figures/registration_examples_acdc/figure_registration_examples_acdc.tex}

\begin{figure}[!ht]
\centering
\includegraphics[width=\linewidth]{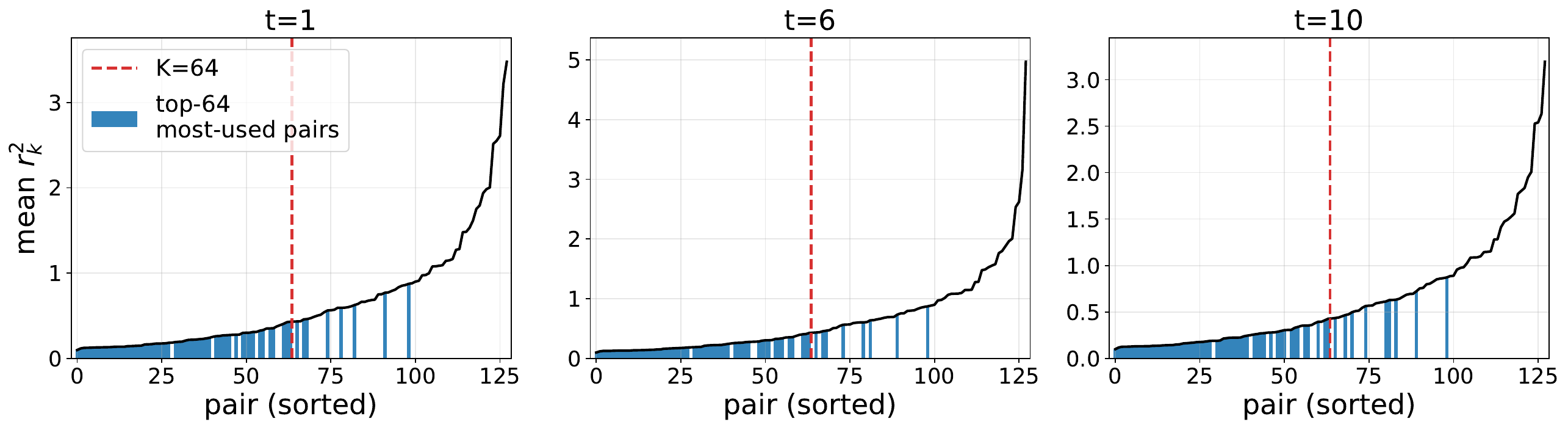}
\caption{Per-pair sorted residuals at three ODE integration steps (first, middle, last) for a held-out ACDC test pair. Black curve: per-pair mean $\bar r_k^2$ sorted ascending. Blue fill: the $K{=}64$ pairs most frequently selected by the top-$K$ mask at this step. Red dashed line: $K$ cutoff. Each panel is independently y-scaled.}
\label{fig:residual_per_pair_acdc}
\end{figure}

\subsection{Echocardiography (CAMUS)}
\label{sec:camus}

We evaluate on CAMUS \citep{leclerc2019camus} ED\(\rightarrow\)ES reporting LV myocardium Dice as in \citet{zhang2024adacs} (Appendix~\ref{app:camus-protocol}). On CAMUS we deploy two cascaded stages of the same architecture, alternating within a single ODE integration: a coarse stage with effective kernel \(k_f=64\) at stride 8 for the first half of steps, then a fine stage with kernel \(k_f=32\) at stride 4 refining against the coarse-warped image.

\begin{table}[h]
\centering
\small
\caption{CAMUS ED\(\rightarrow\)ES registration, LV myocardium Dice. Baselines as reported in
\citet{zhang2024adacs} Tables~1 and~5; ours is reported at seed 42, with cross-seed stability in Appendix~\ref{app:camus-protocol}.}
\label{tab:camus}
\setlength{\tabcolsep}{8pt}
\begin{tabular}{l c}
\toprule
Method                                              & Myo.\ Dice           \\
\midrule
Initial (no registration)                           & \(0.668\)            \\
Elastix (classical)                                 & \(0.802\)            \\
\midrule
DiffuseMorph + MSE                                  & \(0.752 \pm 0.087\)  \\
TransMorph + MSE \citep{chen2022transmorph}         & \(0.792 \pm 0.061\)  \\
VoxelMorph + MSE \citep{balakrishnan2019voxelmorph} & \(0.815 \pm 0.056\)  \\
VoxelMorph + AdaCS \citep{zhang2024adacs}           & \(\mathbf{0.817 \pm 0.054}\) \\
\midrule
Ours (masked, NCC, 2 stages)                        & \(\mathbf{0.820 \pm 0.060}\) \\
\bottomrule
\end{tabular}
\end{table}

Our two-stage architecture matches VoxelMorph + AdaCS at 0.820 Dice vs 0.817 (within standard deviations), without the auxiliary correspondence-scoring estimator or adaptive-smoothness machinery; the residual mechanism described in Section~\ref{sec:method} provides equivalent per-location reliability gating internally. Per-step residual diagnostics on a representative pair appear in Figure~\ref{fig:residual_per_pair_camus_staged}.

\input{figures/registration_examples_camus/figure_registration_examples_camus.tex}

\begin{figure}[!ht]
\centering
\includegraphics[width=\linewidth]{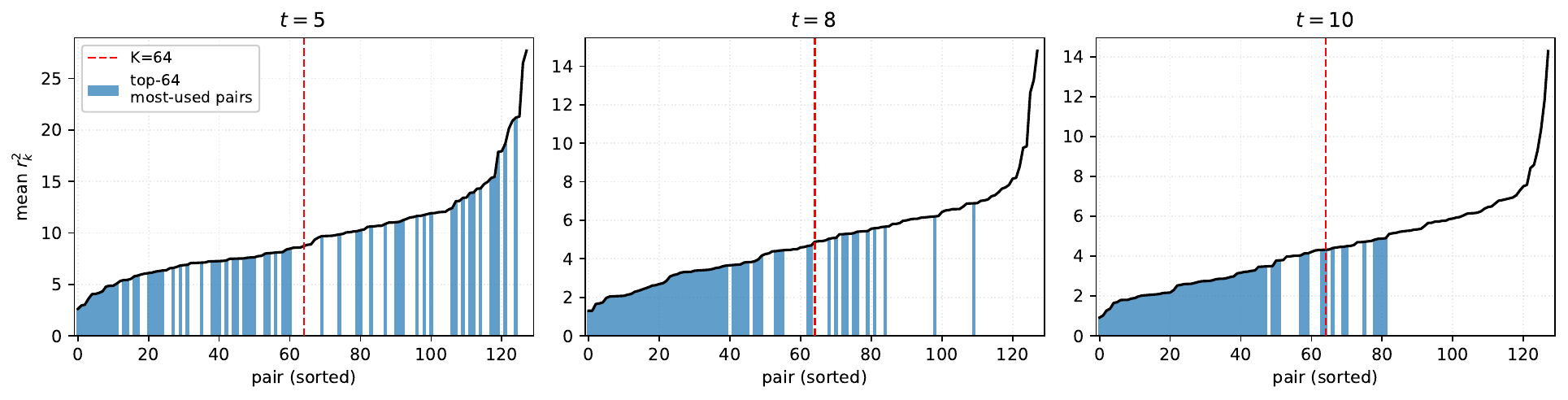}
\caption{Per-pair sorted residuals at the first, middle, and last ODE integration steps of the fine stage for a held out CAMUS test pair (two-stage model, NCC; $C{=}128$ filter pairs). Black curve: per-pair mean $\bar r_k^2$ sorted ascending. Blue fill: the $K{=}64$ pairs most frequently selected by the top-$K$ mask at this step. Red dashed line: $K$ cutoff. Each panel is independently y-scaled.}
\label{fig:residual_per_pair_camus_staged}
\end{figure}

\subsection{Quantum Harmonic Oscillator: Eigenstate recovery from time pairs}
\label{sec:qho}

The image-correspondence experiments above test the framework within a single problem family. To probe the algebraic generality of the primitive (filter pairs spanning 2-D invariant subspaces of an orthogonal operator), we apply the same architecture, with the spatial machinery stripped, to a problem in which the operator is no longer an image warp.

For a time-independent Hamiltonian \(H\) with discrete spectrum (eigenvalue/vector pairs) \(\{E_n, \psi_n\}\), the time-evolution operator
\(U(t) = \mathrm{exp}(-iHt/\hbar)\)
is unitary and acts on each non-degenerate energy eigenspace as a planar rotation \(R(E_nt/\hbar)\in SO(2)\), the same block-diagonal structure of Eq.~\eqref{eq:action} with \(L = U(t)\). The filter-pairs learn the spatial representation of the eigenstates, and the bilinear interaction extracts the per-pair phase \(E_n\tau/\hbar\).

We pick the 1-D quantum harmonic oscillator (QHO) for the cleanest possible test: eigenstates and energies are known analytically (Hermite functions, \(E_n = \hbar\omega(n+\tfrac{1}{2})\)). Training pairs are random unit-norm complex superpositions \(\Psi(0) = \sum_{n=0}^{15} c_n \psi_n^{\mathrm{true}}\) time-evolved analytically to \(\Psi(\tau)\), with each \(c_n\) sampled from a complex normal distribution and \(\tau\) drawn uniformly per pair from \([0.05,\tau_\mathrm{max}]\), \(\tau_\mathrm{max}\in\{0.6,1.7\}\). The model never sees \(H\), the eigenstates, the coefficients, or \(\tau\) (only the wavefunction pair). The architecture consists of \(K=16\) learned 1-D filters \(\boldsymbol{\psi}_k\), a per-filter learned rotation rate \(E_k\), and a small permutation-equivariant readout MLP that recovers \(\tau\); the wavefunction at \(\tau\) is reconstructed in closed form (Appendix~\ref{app:qho}).
\begin{figure}[!ht]
\centering
\includegraphics[width=\linewidth]{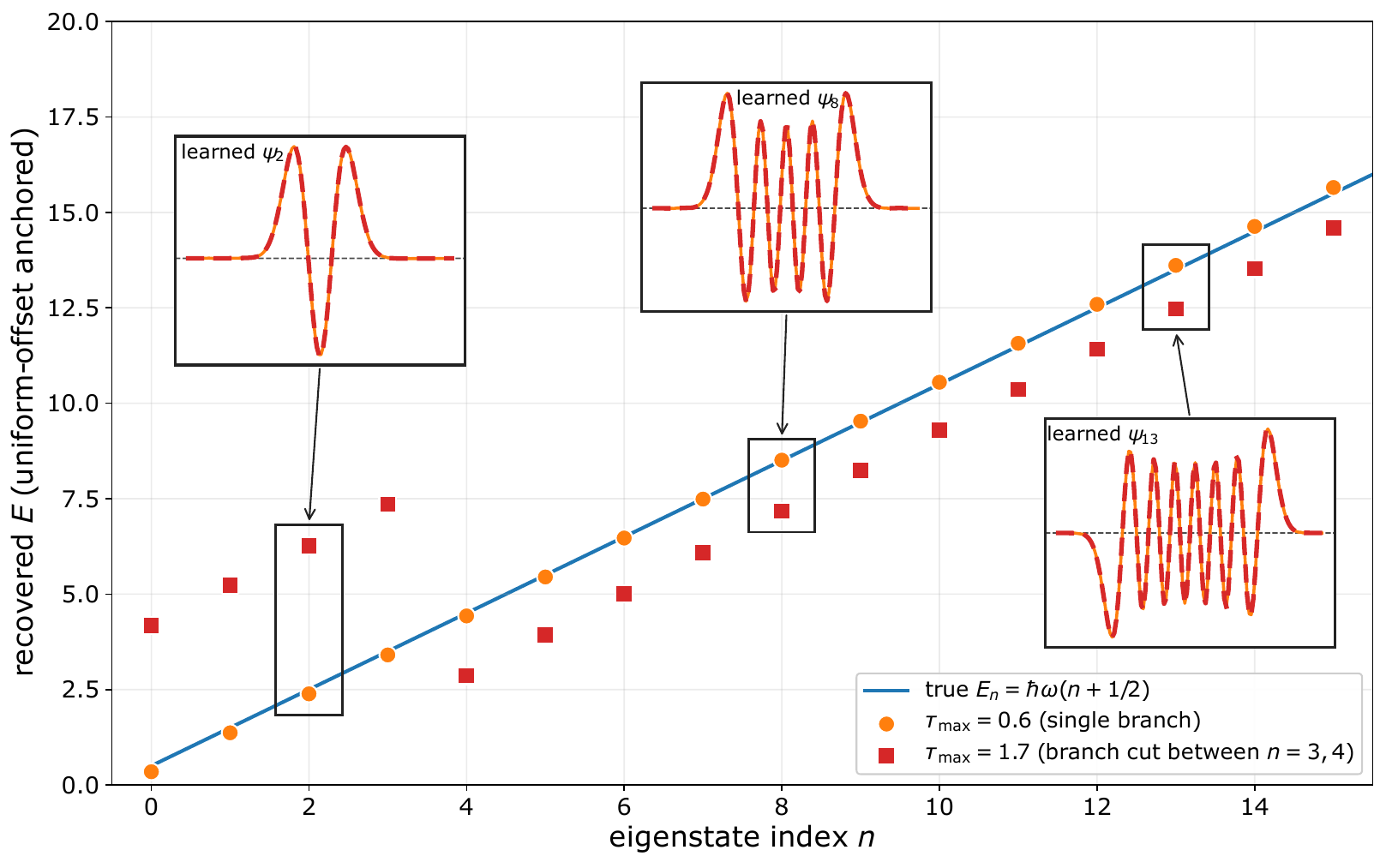}
\caption{Cross-output asymmetry at \(\tau_\mathrm{max}=1.7\). Recovered energy levels show fragmentation into two branches between the \(n=3\) and \(n=4\) states' boundary where \(E_n\tau_\mathrm{max}\) crosses \(2\pi\). Insets show learned eigenstates at representative \(n\); all learned eigenstates coincide with the analytic Hermite eigenstates with overlap 1.000. Full \(E_n\) and \(\psi_n(x)\) recovery in Appendix~\ref{app:qho}.}
\label{fig:qho_eigenstates}
\end{figure}


Figure~\ref{fig:qho_eigenstates} summarizes the results. Learned filters match the analytic Hermite eigenstates pointwise at both \(\tau_{\max}\) regimes (median overlap \(1.000\)). At \(\tau_{\max} = 0.6\), recovered energies match the analytic values up to a single global offset; after calibrating against $E_0$ (read from the single-frequency oscillation of the matrix element $\langle \psi_0 | U(\tau) | \psi_0 \rangle$ (equivalently $\psi_0^T U(\tau) \psi_0$) across pair times, Appendix~\ref{app:qho}; here $E_0 = 0.5$), per-filter absolute energy error falls below \(0.05\) for all \(n\), and the held-out time stamps are recovered with Pearson correlation \(1.000\). From the recovered eigenstates and energies, the truncated Hamiltonian is reconstructed via the spectral decomposition \(\hat{H} = \sum_n E_n^{\mathrm{learned}} |\psi_n^{\mathrm{learned}}\rangle\langle\psi_n^{\mathrm{learned}}|\), with no parameterization of \(H\) or assumed functional form for the operator.

The framework's two outputs, eigenstate recovery and eigenvalue recovery, exhibit different sensitivities to the time-interval distribution. Eigenstate recovery is \(\tau\)-blind: learned filters continue to match the analytic eigenstates pointwise at both \(\tau_{\max} = 0.6\) and \(\tau_{\max} = 1.7\), with \(\tau_{\max} = 1.7\) in the multi-wrap regime where \(E_n \tau_{\max} > 2\pi\) for most filters. The structural primitive that identifies invariant subspaces (the residual) does not depend on \(\tau\), so the eigenvectors are unaffected by how aggressively the time interval is sampled. Eigenvalue recovery has a sharper, per-filter boundary at \(E_n \tau_{\max} < 2\pi\) (Figure~\ref{fig:qho_eigenstates}): at \(\tau_{\max} = 0.6\) the ladder remains clean; at \(\tau_{\max} = 1.7\) the spectrum fragments into two branches at the \(n=3\) / \(n=4\) boundary where \(E_n \tau_{\max}\) crosses \(2\pi\). The mechanism is gauge-mediated: the reconstruction objective, NCC, is invariant to per-pair global phase, so filters with different wrap counts can settle on different unwrap branches without penalty. Eigenstate insets in Figure~\ref{fig:qho_eigenstates} show that the same training run that fragments the spectrum recovers \(\psi_2\) (below the branch cut), \(\psi_8\) and \(\psi_{13}\) (both above) pointwise.

\section{Conclusion}
\label{sec:conclusion}
We introduced a learned generalization of phase correlation that extends the primitive from the global-translation regime in which it has always operated \citep{kuglin1975phasecorr} to dense, spatially-varying non-rigid deformation, and demonstrated its algebraic generality on unitary time-evolution where the same architecture recovers Hamiltonian eigenstates and energies from observation pairs alone. The per-pair residual and filter-bank visualization expose quantities interpretable within the model's algebraic structure. The same algebraic ingredient that solves cardiac registration recovers a quantum Hamiltonian.

The structural commitment fixes the form of the answer in each setting. What the network learns is what the structure cannot determine analytically: which subspaces are locally invariant, how to assemble per-subspace rotations into a coherent output, and how to absorb the approximation gap when the operator is only locally orthogonal. The descriptor-based paradigm asks the network to discover correspondence from scratch; here, the network is responsible only for the residual irregularities of a structure that already encodes most of the task. Limitations are discussed in Appendix~\ref{app:limitations}.


\small
\bibliographystyle{unsrtnat}
\bibliography{references}

\appendix

\section{Classical phase correlation}
\label{app:phase-correlation}

Phase correlation \citep{kuglin1975phasecorr} is a Fourier-domain technique for estimating the global translation between two signals. We illustrate the algorithm in the continuous 1-D setting for notational simplicity.

\paragraph{Setup.}
Let \(f, g : \mathbb{R} \to \mathbb{R}\) be two functions related by a global translation \(c \in \mathbb{R}\), i.e.\ \(g(x) = f(x - c)\). The goal is to recover \(c\) from the pair \((f, g)\) alone.

\paragraph{Fourier shift theorem.}
With the convention \(\hat{f}(k) = \int_{-\infty}^{\infty} f(x)\, e^{-ik x}\, dx\) and \(f(x) = \int_{-\infty}^{\infty} \hat{f}(k)\, e^{ik x}\, dk\) for the Fourier and inverse Fourier transforms of \(f\), respectively, we have
\begin{align*}
\hat{g}(k) &= \int_{-\infty}^{\infty} g(x)\, e^{-ik x}\, dx \\
&= \int_{-\infty}^{\infty} f(x-c)\, e^{-ik x}\, dx \\
&= \int_{-\infty}^{\infty} f(u)\, e^{-ik (u+c)}\, du \;\;\;\;\mathrm{for}\;\; u = x - c \\
&= e^{-ikc}\int_{-\infty}^{\infty} f(u)\, e^{-ik u}\, du \\
&= \hat{f}(k)\, e^{-ikc}.
\end{align*}
The translation is encoded entirely in the spatial frequency phase of \(\hat{g}\) relative to \(\hat{f}\).

\paragraph{Cross-power spectrum and phase normalization.}
Multiplying \(\hat{g}(k)\) by the complex conjugate \(\overline{\hat{f}(k)}\) yields
\[
\overline{\hat{f}(k)}\, \hat{g}(k) = |\hat{f}(k)|^2\, e^{-ikc}.
\]
The amplitude factor \(|\hat{f}(k)|^2\) depends on the spectral content of \(f\); only the phase \(e^{-ikc}\) carries the translation. Normalizing by magnitude removes the amplitude and leaves a pure phase,
\[
A(k) \;=\; \frac{\overline{\hat{f}(k)}\, \hat{g}(k)}{|\overline{\hat{f}(k)}\, \hat{g}(k)|} \;=\; e^{-ikc}.
\]

\paragraph{Peak localization.}
The inverse Fourier transform of \(A\) is a Dirac delta at the displacement,
\[
\int_{-\infty}^{\infty} e^{-ikc}\, e^{ik x}\, dk \;=\; \int_{-\infty}^{\infty} e^{ik(x-c)}\, dk \;=\; \delta(x - c).
\]
Thus applying \(\arg\max\) to \(\delta(x - c)\) or integrating \(\int_{-\infty}^{\infty} x\, \delta(x - c)\, dx\) recovers \(c\) exactly when the global-translation assumption holds.

\paragraph{Vectorized discrete formulation.}
The same construction in matrix form, with \(\mathbf{f}, \mathbf{g} \in \mathbb{R}^N\) the discrete samples of \(f\) and \(g\) on a grid of \(N\) points. The Fourier transform becomes the \(N \times N\) DFT matrix \(\mathbf{F}\) with entries \(\mathbf{F}_{nm} = e^{-i 2\pi nm/N}/\sqrt{N}\); on real inputs its output satisfies the conjugate-symmetry relation \((\mathbf{F}\mathbf{f})_{N-n} = \overline{(\mathbf{F}\mathbf{f})_n}\). The translation \(g(x) = f(x - c)\) becomes
\[
\mathbf{g} \;=\; \mathbf{T}_c\, \mathbf{f}, \qquad \mathbf{T}_c \;=\; \mathbf{F}^{-1}\, \mathbf{D}_c\, \mathbf{F},
\]
where \(\mathbf{D}_c\) is diagonal with \(\mathbf{D}_{c,nn} = e^{-i k_n c}\) and \(k_n = 2\pi n/N\). The factorization is the matrix form of the \emph{FFT \(\to\) phase shift \(\to\) IFFT} algorithm.

\paragraph{Translation is orthogonal.}
On real inputs \(\mathbf{T}_c\) is real-orthogonal: \(\mathbf{T}_c^T \mathbf{T}_c = \mathbf{I}\). Orthogonal matrices are the real-valued counterpart of unitary matrices, with each complex-conjugate eigenvalue pair acting as a real \(2 \times 2\) rotation block on the corresponding 2-D real subspace \citep{gelfand1989linear}; here, the conjugate pair \((e^{-i k_n c}, e^{+i k_n c})\) on the diagonal of \(\mathbf{D}_c\) becomes \(R(k_n c) \in SO(2)\) on the cosine/sine pair at frequency \(k_n\). \(\mathbf{T}_c\) thus has the same block-diagonal-with-rotations structure as the patch operator \(\mathbf{L}_p\) in Eq.~\eqref{eq:action}, and is the formal counterpart of the ``permutation matrices are orthogonal'' observation in Section~\ref{sec:method}.

\paragraph{Spatially-varying displacement.}
When \(c\) is constant, \(\mathbf{T}_c\) is block-diagonal in the spatial harmonic basis: a single \(R(k_n c)\) at every frequency. When \(c\) varies with \(x\),  \(\,g(x) = f(x - c(x))\), the substitution that pulled \(c\) out of the integral fails: \(u = x - c(x)\) no longer has a constant Jacobian, and
\[
\hat{g}(k) \;=\; \int_{-\infty}^{\infty} f(x - c(x))\, e^{-ikx}\, dx \;\neq\; \hat{f}(k)\, e^{-ikc}.
\]
The matrix counterpart is that \(\mathbf{T}_{c(\cdot)}\) is still real-orthogonal, but is no longer block-diagonal in the spatial harmonic basis: the action couples frequencies.

\paragraph{Learning the right basis.}
Spatially-varying displacement remains a real-orthogonal operation at each location --- still block-diagonalizable into \(2 \times 2\) rotations, just not by the spatial harmonic basis. Neural Phase Correlation seeks the basis \(\mathbf{U}\) in which the local transformation \(\mathbf{L}_p\) \emph{is} block-diagonal with \(2 \times 2\) rotation blocks (Eq.~\eqref{eq:action}). Classical phase correlation is the special case in which \(\mathbf{U}\) is the spatial harmonic basis, optimal only when \(c\) is constant. For spatially-varying \(c(x)\), the diagonalizing basis depends on \(c\) itself; the learned filter pairs \(\mathbf{\Psi}, \mathbf{\Phi}\) are the vehicle for recovering it from data.

\section{Residual}
\label{app:residual}

This appendix formalizes the basis for interpreting the per-location, per-subspace residual \(r_k^2(p)\) of Section~\ref{sec:method} as an inference-time validity signal for \(\mathrm{span}(\psi_k)\) under the local orthogonal operator \(L_p\). We first analyze the residual in the idealized scenario in which both filter pairs span a common invariant subspace of \(L_p\), establishing its invariance and swap symmetry. The realistic case in which alignment may fail is then characterized by Lemma~\ref{lem:necessity} and Theorem~\ref{thm:sufficiency}.

\paragraph{Idealized scenario: two filter pairs, three rotations.}
The main-text residual is \(r_k^2 = \bigl(\|\boldsymbol{\phi}_k^T y\| - \|\boldsymbol{\psi}_k^T x\|\bigr)^2\), with separate filter pairs for the moving image (\(\boldsymbol{\psi}_k\)) and the fixed image (\(\boldsymbol{\phi}_k\)). Suppose both pairs span the \(k\)-th 2-D eigenspace of \(L_p\) with orthonormal columns:
\[
\boldsymbol{\psi}_k = U_k\,R(\theta_\psi), \qquad \boldsymbol{\phi}_k = U_k\,R(\theta_\phi),
\]
where \(U_k = [u_{2k-1}, u_{2k}]\) is the \(k\)-th 2-D eigenspace of \(L_p\), with \(L_p|_{U_k} = R(\theta_k) \in SO(2)\), and \(R(\theta_\psi), R(\theta_\phi) \in SO(2)\) are the in-plane orientations of the two filter pairs relative to that eigenframe.\footnote{Restricting to \(SO(2)\) rather than \(O(2)\) excludes a reflection of each filter frame; the norm-equality argument below goes through identically in the reflected case, so the residual's vanishing condition is unaffected.} Three rotations are then in play, all in \(SO(2)\):
\begin{itemize}
\setlength{\itemsep}{0pt}\setlength{\parskip}{0pt}
  \item \(R(\theta_\psi)\) and \(R(\theta_\phi)\): the orientations of the moving and fixed filter pairs in the eigenframe \(U_k\).
  \item \(R(\theta_\psi - \theta_\phi)\): the inter-filter rotation between the moving and fixed pairs, parameterized by a single angle difference.
\end{itemize}

\paragraph{Norm equality and swap symmetry.}
Let \(p_x = U_k^T x\) be the \(U_k\)-projection of \(x\) in eigenframe coordinates, and write \(y = L_p x\). Each of the four projection norms reduces to \(\|p_x\|\):
\begin{align*}
\|\boldsymbol{\psi}_k^T x\| &= \|R(-\theta_\psi)\,p_x\| = \|p_x\|, &
\|\boldsymbol{\phi}_k^T y\| &= \|R(-\theta_\phi)\,R(\theta_k)\,p_x\| = \|p_x\|, \\
\|\boldsymbol{\psi}_k^T y\| &= \|R(-\theta_\psi)\,R(\theta_k)\,p_x\| = \|p_x\|, &
\|\boldsymbol{\phi}_k^T x\| &= \|R(-\theta_\phi)\,p_x\| = \|p_x\|,
\end{align*}
since each is a planar rotation applied to \(p_x\) and \(SO(2)\) preserves the Euclidean norm. Hence \(r_k^2 = 0\), and the swap \((\boldsymbol{\psi}_k, \boldsymbol{\phi}_k) \to (\boldsymbol{\phi}_k, \boldsymbol{\psi}_k)\) gives \(\bigl(\|\boldsymbol{\psi}_k^T y\| - \|\boldsymbol{\phi}_k^T x\|\bigr)^2 = 0\) by the same calculation: the residual is invariant under exchange of moving and fixed filter pairs in the ideal case.

\paragraph{Setup.}
Let \(L \in O(n)\) with \(n = k_f^2\), so \(L^T L = I\), and assume each filter pair is orthonormal,
\begin{equation*}
\boldsymbol{\psi}_k^T \boldsymbol{\psi}_k \;=\; \boldsymbol{\phi}_k^T \boldsymbol{\phi}_k \;=\; I_2,
\end{equation*}
and that the two pairs share a common span: \(\mathrm{span}(\boldsymbol{\psi}_k) = \mathrm{span}(\boldsymbol{\phi}_k)\), equivalently \(\boldsymbol{\phi}_k = \boldsymbol{\psi}_k R\) for some \(R \in O(2)\) (this is the idealized structure; its failure is precisely what the residual measures). The two projectors then coincide,
\begin{equation*}
P \;=\; \boldsymbol{\psi}_k\boldsymbol{\psi}_k^T \;=\; \boldsymbol{\phi}_k\boldsymbol{\phi}_k^T \;\in\; \mathbb{R}^{n \times n},
\end{equation*}
since \(\boldsymbol{\phi}_k\boldsymbol{\phi}_k^T = \boldsymbol{\psi}_k R R^T \boldsymbol{\psi}_k^T = \boldsymbol{\psi}_k\boldsymbol{\psi}_k^T\). For a patch \(x \in \mathbb{R}^n\) and its transformed counterpart \(y = L x\), the moving projection is \(\boldsymbol{\psi}_k^T x \in \mathbb{R}^2\), the fixed projection is \(\boldsymbol{\phi}_k^T y \in \mathbb{R}^2\), and
\begin{equation*}
r^2 \;=\; \bigl(\|\boldsymbol{\phi}_k^T y\| - \|\boldsymbol{\psi}_k^T x\|\bigr)^2.
\end{equation*}

\paragraph{Invariance condition.}
\(\mathrm{span}(\boldsymbol{\psi}_k)\) is invariant under \(L\), i.e.\ \(L\,\mathrm{span}(\boldsymbol{\psi}_k) \subseteq \mathrm{span}(\boldsymbol{\psi}_k)\), if and only if the projector \(P\) commutes with \(L\):
\begin{equation*}
[P, L] \;:=\; P L - L P \;=\; 0.
\end{equation*}
Under this condition the restriction \(L|_{\mathrm{span}(\boldsymbol{\psi}_k)}\) is itself orthogonal: for any \(u, v \in \mathrm{span}(\boldsymbol{\psi}_k)\), \(\langle Lu, Lv \rangle = \langle u, v \rangle\) since \(L \in O(n)\) preserves inner products globally, so its restriction inherits the property. On the 2-D subspace \(\mathrm{span}(\boldsymbol{\psi}_k)\), \(L\) therefore acts as a planar rotation \(R(\theta)\) for some \(\theta\) (the \(SO(2)\) case) or a rotation composed with a reflection (the \(O(2) \setminus SO(2)\) case); both preserve the Euclidean norm.

\begin{lemma}[Necessity]
\label{lem:necessity}
\emph{If \(\mathrm{span}(\boldsymbol{\psi}_k)\) is invariant under \(L\) (equivalently \([P,L] = 0\)), then \(r^2(x) = 0\) for every \(x \in \mathbb{R}^n\).}
\end{lemma}

\begin{proof}
Since \(L\) is orthogonal, \(L^T = L^{-1}\). Expanding using \(P = \boldsymbol{\phi}_k\boldsymbol{\phi}_k^T = \boldsymbol{\psi}_k\boldsymbol{\psi}_k^T\),
\begin{equation*}
\|\boldsymbol{\phi}_k^T y\|^2 \;=\; y^T P\, y \;=\; x^T L^T P L\, x \;=\; x^T L^{-1} P L\, x.
\end{equation*}
The commutativity \(P L = L P\) gives \(L^{-1} P L = P\), so
\begin{equation*}
\|\boldsymbol{\phi}_k^T y\|^2 \;=\; x^T P\, x \;=\; \|\boldsymbol{\psi}_k^T x\|^2.
\end{equation*}
Taking square roots, \(\|\boldsymbol{\phi}_k^T y\| = \|\boldsymbol{\psi}_k^T x\|\) and hence \(r^2 = 0\).
\end{proof}

\begin{theorem}[Almost-sure sufficiency]
\label{thm:sufficiency}
\emph{Suppose \(\mathrm{span}(\boldsymbol{\psi}_k)\) is not invariant under \(L\), i.e., \([P, L] \ne 0\). If \(x\) is drawn from an absolutely continuous distribution on \(\mathbb{R}^n\) independently of \(L\), then}
\begin{equation*}
\mathbb{P}\bigl(r^2(x) = 0\bigr) \;=\; 0.
\end{equation*}
\end{theorem}
\begin{proof}
The residual vanishes iff \(\|\boldsymbol{\phi}_k^T y\|^2 \;=\; \|\boldsymbol{\psi}_k^T x\|^2\), i.e.,
\begin{equation*}
x^T A\, x \;=\; 0, \qquad A \;:=\; L^{-1} P L - P.
\end{equation*}

\emph{\(A\) is symmetric.} Using \(P^T = P\), \(L^T = L^{-1}\), and \((L^{-1})^T = L\),
\begin{equation*}
A^T \;=\; L^T P^T (L^{-1})^T - P^T \;=\; L^{-1} P L - P \;=\; A.
\end{equation*}

\emph{\(A \ne 0\).} Left-multiplying \(A\) by \(L\) yields
\begin{equation*}
L A \;=\; P L - L P \;=\; [P, L].
\end{equation*}
Since \(L\) is invertible, \(A = 0 \iff [P, L] = 0\). By hypothesis \([P, L] \ne 0\), hence \(A \ne 0\).

Thus \(x \mapsto x^T A x\) is a nonzero quadratic form on \(\mathbb{R}^n\). The
set where it vanishes, \(\{x : x^T A x = 0\}\), is the zero set of a nonzero
polynomial in \(n\) real variables, and the zero set of
any nonzero polynomial on \(\mathbb{R}^n\) has Lebesgue measure zero.
Any continuous probability distribution on \(\mathbb{R}^n\) assigns probability zero to a measure-zero set, hence \(\mathbb{P}(x^T A x = 0) = 0\).
\end{proof}

\paragraph{Structure of the failure set.}
The vanishing set above can be described more concretely. Both \(P\) and
\(L^{-1} P L\) are rank-2 projections, so their difference \(A\) has rank at most
four: the value of \(x^T A x\) depends only on the components of \(x\) along at
most four particular directions in \(\mathbb{R}^n\), and is identically zero
along the other \(n - 4\).

\section{Architecture and implementation details}
\label{app:architecture}

This appendix gives the full description of each stage of the architecture summarized in Section~\ref{sec:method} and Figure~\ref{fig:pipeline}.

\subsection{Stage 1: Learned filter pairs}

The orthogonal-operator model \(y(p) = L_p\, x(p)\) (Eq.~\eqref{eq:action}) defines the setting; \(L_p\) admits a canonical decomposition into \(2 \times 2\) rotation blocks that organize our generalization of the per-frequency structure of phase correlation. Because the invariant subspaces of a general orthogonal \(L_p\) depend on \(L_p\) itself, the fixed Fourier basis used by classical phase correlation is replaced with two banks of \(C\) learnable filter pairs, \(\psi_k\) and \(\phi_k\) for \(k = 1, \ldots, C\), applied to the moving and fixed images via strided convolution (each in \(\mathbb{R}^{n \times 2}\)). The sole architectural constraint is unit \(\ell_2\) norm on each filter; orthonormality within a pair and alignment of its span with an invariant subspace of \(L_p\) are not imposed architecturally; they emerge from training against the task loss. Figure~\ref{fig:overview} shows representative learned filter pairs on cardiac MRI.

Under the ideal structure (\(\mathrm{span}(\psi_k)\) invariant under \(L_p\), pair orthonormal), the restriction \(L_p|_{\mathrm{span}(\psi_k)}\) is a planar rotation \(R(\theta_k(p)) \in SO(2)\). Each filter pair plays the role of a single classical frequency: it identifies a 2-D subspace on which the local transformation acts as a rotation by an angle \(\theta_k(p)\). \citet{memisevic2010factored} established empirically that filter pairs in the factored bilinear architecture converge to Fourier harmonic pairs when trained on pure-translation data, so the Fourier-filter limit case in our reduction is realized by training rather than imposed.

\subsection{Stage 2: Bilinear interaction on filter responses}

Classical phase correlation forms the cross-power spectrum \(\hat y(k) \cdot \overline{\hat x(k)}\) by multiplying per-frequency Fourier coefficients. Under our modeling assumption, the analogous operation is a \emph{bilinear interaction on the filter-response subspace}, evaluated independently at each coarse-grid location \(p\). Let \(x_p, y_p \in \mathbb{R}^n\) denote the moving and fixed patches centered at \(p\). The per-subspace bilinear interaction is
\begin{equation}
z_k(p) \;=\; \begin{pmatrix}
(\psi_k^T x_p)^T \, (\phi_k^T y_p) \\[2pt]
(\psi_k^T x_p)^T J\, (\phi_k^T y_p)
\end{pmatrix},
\qquad J = \begin{pmatrix} 0 & -1 \\ 1 & 0 \end{pmatrix},
\label{eq:interaction}
\end{equation}
where \(\psi_k^T x_p, \phi_k^T y_p \in \mathbb{R}^2\) are the projections of the moving and fixed patches onto their respective filter pairs. Each component of \(z_k(p)\) is a bilinear form in \((\psi_k^T x_p, \phi_k^T y_p)\): the symmetric inner product, and the skew-symmetric inner product induced by the planar 90-degree rotation \(J\). The pair-product is computed entirely in the encoded filter-response space and never returns to image space; this is what we refer to as bilinear interaction.

Under \(\mathbb{R}^2 \cong \mathbb{C}\) (encoding \(v = (v_1, v_2) \in \mathbb{R}^2\) as \(v_1 + i v_2 \in \mathbb{C}\)), \(z_k\) is \(\overline{\psi_k^T x_p} \cdot \phi_k^T y_p\), the real-vector form of the classical cross-power product, and its two components are proportional to \((\cos \theta_k, \sin \theta_k)\). The decoder receives a polar factorization: unit-circle direction (the per-subspace rotation angle) and magnitude separately \((\cos\theta_k, \sin\theta_k, |z_k|)\). Magnitude normalization inherits the classical phase-correlation cancellation of intensity bias; keeping magnitude as a separate channel lets downstream weighting attenuate a subspace's contribution without corrupting its angle estimate.

\subsection{Stage 3: Positional-encoding decoder}

Classical phase correlation completes the pipeline by inverse-transforming the normalized cross-power spectrum to produce a 2-D plane with a Dirac peak at the displacement, then localizing that peak by \(\arg\max\). We instead interpret the displacement vector as integrating the product of the nearly-Dirac peak and the coordinate grid. The decoder implements this interpretation through two intermediate spaces: a \emph{gate space} \(\mathbb{R}^G\) in which per-pair features and positional encodings interact multiplicatively, and the \emph{PE space} \(\mathbb{R}^{n_\mathrm{pe}}\) in which the displacement residual lives. In our released configurations the gate-space dimension equals the filter-pair count (\(G = C\)), though the two spaces are structurally distinct: \(\mathbb{R}^C\) is indexed by filter pair, while \(\mathbb{R}^G\) is the space in which pair features and positional encodings combine.

Let \(\boldsymbol{\gamma} \in \mathbb{R}^{n_\mathrm{pe}}\) denote a NeRF-styled sinusoidal positional encoding of the coordinate grid \citep{mildenhall2020nerf}. The masked bilinear interaction at the current coarse-grid location yields a per-pair polar feature vector \(z \in \mathbb{R}^{3C}\), stacking the triple \((\cos\theta_k, \sin\theta_k, |z_k|)\) over \(k = 1, \ldots, C\). A two-layer \(1\!\times\!1\) MLP \(\mathrm{MLP}_z : \mathbb{R}^{3C} \to \mathbb{R}^G\) with hidden width \(B\) and a LeakyReLU between layers projects \(z\) into the gate space, while a single learned linear map \(W_\gamma \in \mathbb{R}^{G \times n_\mathrm{pe}}\) lifts \(\boldsymbol{\gamma}\) into the same space; the two are multiplied element-wise (resembling the delta peak's product with the coordinate grid in classical phase correlation), a second MLP maps the gated vector back to PE space, and a final readout produces the Cartesian displacement:
\begin{equation}
\begin{aligned}
m \;&=\; \mathrm{MLP}_z(z),                                                                 & \boldsymbol{\gamma}_g \;&=\; W_\gamma\, \boldsymbol{\gamma}, \\
\delta\boldsymbol{\gamma} \;&=\; \mathrm{MLP}_\mathrm{pe}\bigl(\, m \odot \boldsymbol{\gamma}_g \,\bigr) \;-\; \boldsymbol{\gamma},  & \delta\varphi \;&=\; \mathrm{head}\bigl(\delta\boldsymbol{\gamma}\bigr),
\end{aligned}
\label{eq:decoder}
\end{equation}
where \(m, \boldsymbol{\gamma}_g \in \mathbb{R}^G\), \(\delta\boldsymbol{\gamma} \in \mathbb{R}^{n_\mathrm{pe}}\), \(\odot\) is element-wise multiplication, and \(\mathrm{MLP}_\mathrm{pe} : \mathbb{R}^G \to \mathbb{R}^{n_\mathrm{pe}}\) is a two-layer \(1\!\times\!1\) net with hidden width \(G\) and a LeakyReLU between layers (Table~\ref{tab:impl_arch}). The residual \(\delta\boldsymbol{\gamma}\) is the displacement representation in PE space; \(\mathrm{head}\) maps it to a Cartesian displacement vector. The canonical readout is a single linear \(\mathrm{Conv}_{1\times 1}(n_\mathrm{pe}, 2)\); in this paper we use a depth-1 mini-UNet of the same input/output shape with \(3\!\times\!3\) spatial kernel size, which couples neighboring spatial indices for additional regularization.

\subsection{Training objective}

All checkpoints are trained unsupervised against a similarity loss between the warped moving and fixed images, plus a smoothness regularizer on the displacement field and a Jacobian-determinant fold penalty for topology preservation. The full loss specification is reported per benchmark in Table~\ref{tab:impl_train}. The ODE integration scheme used at inference is the one defined by Eq.~\eqref{eq:ode}.

\subsection{Implementation details}
\label{app:implementation}

The conceptual architecture is described in the preceding subsections of this appendix. Both released checkpoints (ACDC, CAMUS) instantiate the same UNet decoder architecture and differ only in a small number of capacity, regularization, and schedule choices appropriate to each modality. Across both benchmarks the encoder is a normalized convolution producing two banks of $C$ filter pairs at coarse-grid stride $s$; per-pair features and positional encodings interact in a gate space of dimension $G$ before being mapped back to the PE space and read out as a Cartesian displacement. The depth-1 mini-UNet uses a hidden channel dimension of \(16\). Reported numbers are per-pair test-set mean $\pm$ 1-$\sigma$ sample standard deviation from seed 42; cross-seed stability across seeds 42, 43, 44 is reported in the per-benchmark appendices.

\begin{table}[h]
\centering
\small
\setlength{\tabcolsep}{8pt}
\caption{Architecture hyperparameters per benchmark.}
\label{tab:impl_arch}
\begin{tabular}{lcc}
\toprule
                                     & ACDC          & CAMUS \\
\midrule
Image size                           & $128\times128$ & $128\times128$ \\
Encoder kernel $k_f$                 & $16$           & $64~\&~32$ \\
Encoder stride $s$                   & $4$            & $8~\&~4$ \\
Coarse-grid resolution               & $32\times32$   & $32\times32$ \\
Filter-pair bank $C$                 & $128$          & $128$ \\
Top-$K$ mask                         & $64$           & $64$ \\
Gate-space dim $G$                   & $128$          & $128$ \\
Encoder MLP hidden $B$               & $128$          & $128$ \\
Positional-encoding channels $n_\mathrm{pe}$ & $256$  & $256$ \\
Integration steps $n_t$              & $10$           & $10$ \\
\bottomrule
\end{tabular}
\end{table}

\begin{table}[h]
\centering
\small
\setlength{\tabcolsep}{8pt}
\caption{Loss and optimization hyperparameters per benchmark. Loss is $\mathcal{L} = \mathcal{L}_\mathrm{sim} + \lambda_\mathrm{reg}\mathcal{L}_\mathrm{reg} + \lambda_J\mathcal{L}_{|J|} + \lambda_{\log J}\mathcal{L}_{\log J}$. All checkpoints use Adam.}
\label{tab:impl_train}
\begin{tabular}{lcc}
\toprule
                                       & ACDC                        & CAMUS \\
\midrule
$\mathcal{L}_\mathrm{sim}$             & NCC (window $17$)           & NCC (window $67$) \\
Regularizer $\mathcal{L}_\mathrm{reg}$ & $\|\nabla\varphi\|^2$       & $\|\nabla\varphi\|^2$ \\
$\lambda_\mathrm{reg}$                 & $5\!\times\!10^{-2}$        & $2\!\times\!10^{-1}$ \\
$\lambda_J$ (Jac.\ det.)               & $100$                       & $1$ \\
$\lambda_{\log J}$                     & $10^{-5}$                   & $10^{-2}$ \\
Learning rate                          & $10^{-3}$                   & $10^{-3}$ \\
LR schedule                            & exp.\ $\gamma\!=\!0.997$    & exp.\ $\gamma\!=\!0.996$ \\
Batch size                             & $50$                        & $50$ \\
Epochs                                 & $400$                       & $300$ \\
Weight decay                           & $0$                         & $10^{-5}$ \\
Gradient clip (norm)                   & off                         & off \\
Augmentation                           & h/v flip w/ $p\!=\!0.5$     & none \\
Pairs per epoch                        & full train set              & full train set \\
Seeds reported                         & $42, 43, 44$                & $42, 43, 44$ \\
\bottomrule
\end{tabular}
\end{table}

\begin{table}[h]
\centering
\small
\setlength{\tabcolsep}{6pt}
\caption{Per-benchmark dataset protocol. Splits and pre-processing follow the original published protocols.}
\label{tab:impl_protocol}
\begin{tabular}{lcc}
\toprule
                                & ACDC                            & CAMUS \\
\midrule
Source protocol                  & \citet{jian2025disentangling}    & \citet{zhang2024adacs} \\
Patient split (train/val/test)  & $80 / 20 / 50$                  & $300 / 100 / 100$ \\
Pairs                            & ES--ED, both directions          & ES--ED, both 2CH \& 4CH views \\
In-plane resampling              & as in source                     & $1.0$\,mm \\
Crop / canvas                    & $128\!\times\!128$ ROI on seg.\ bbox & $128\!\times\!128$ center-crop \\
Evaluation metric                & Dice (RV / Myo / LV)            & LV-myocardium Dice \\
\bottomrule
\end{tabular}
\end{table}

\section{Per-benchmark protocol details}
\label{app:protocols}

This appendix gives data-, loss-, and evaluation-specific details for each of the three benchmarks. Architecture, optimizer, and per-benchmark hyperparameters are summarized in Tables~\ref{tab:impl_arch}--\ref{tab:impl_protocol}; the prose below specifies what those tables do not (split rules, evaluation metric definitions, cross-seed stability, comparison caveats). All datasets are publicly distributed for research use: ACDC \citep{bernard2018acdc} and CAMUS \citep{leclerc2019camus} via CREATIS (free registration), OASIS-1 \citep{marcus2007oasis} under the OASIS Data Use Agreement.

\subsection{Cardiac MRI (ACDC)}
\label{app:acdc-protocol}

\paragraph{Data and split (\citet{jian2025disentangling} \S4.1.3).}
We use all 150 labeled ACDC subjects, split 80 train / 20 validation / 50 test at the patient level (seed 2023), yielding 506 / 190 / 194 annotated 2-D short-axis slices respectively (\citep{jian2025disentangling} reports 194 test pairs, but we obtained 394 which is more consistent with the split ratios and assume their test split is a subset of ours). We keep slices in which both ED and ES segmentations contain all three structures (RV, myocardium, LV), reproducing \citet{jian2025disentangling}'s reported train+val slice count of 506 + 190 = 696. Each slice is cropped to a 128\(\times\)128 region of interest centered on the union of the ED/ES segmentation bounding boxes and intensities are min-max normalized to \([0, 1]\) per slice.

\paragraph{Bidirectional training.}
Each iteration performs both forward (ED\(\rightarrow\)ES) and backward (ES\(\rightarrow\)ED) registrations through the same model, with the total loss averaged across directions, matching \citet{jian2025disentangling} \S4.3.

\paragraph{Evaluation.}
Both directions are evaluated separately at test time. We report Dice over the three annotated structures (RV, myocardium, LV) per direction and as the across-direction average. SDlogJ (\(\times 10^{-2}\)) is the standard deviation of \(\log(\det J + 3)\) following \citet{jian2025disentangling}'s implementation, lower indicating smoother volume changes. NDV (\(/\)10\,k) is the rate of pixels with non-positive Jacobian determinant per 10{,}000 foreground voxels (foreground = union of moving and fixed segmentations), lower indicating more diffeomorphic warps.

\paragraph{Loss-control breakdown.}
Under the NCC recipe above, our masked-residual architecture reaches 0.843 Dice on ED\(\rightarrow\)ES, while VoxelMorph re-trained under the identical recipe reaches 0.817; the +0.026 gap is the architectural contribution. The same VoxelMorph re-training shows +0.029 NCC-over-MSE (0.788\(\rightarrow\)0.817), and our own architecture shows a comparable +0.024 NCC-over-MSE (0.819\(\rightarrow\)0.843). Both architectures gain similarly from NCC, so NCC does not asymmetrically favor our method.

\paragraph{Cross-seed stability (Table~\ref{tab:acdc_ncc}).}
The Table~\ref{tab:acdc_ncc} numbers are reported at seed 42. Across three independent training seeds (42, 43, 44) under the recipe above, the cross-seed standard deviations are 0.001, 0.003, and 0.002 on Dice (ED\(\rightarrow\)ES, ES\(\rightarrow\)ED, and the across-direction average respectively), 0.16 on SDlogJ, and 2.9 on NDV. The reported deltas over DWCP \citep{jian2025disentangling} therefore exceed cross-seed variability by an order of magnitude in each direction.

\subsection{Echocardiography (CAMUS)}
\label{app:camus-protocol}

\paragraph{Data and split.}
The CAMUS dataset \citep{leclerc2019camus} comprises 500 patients with paired apical two- and four-chamber views annotated at ED and ES with LV myocardium segmentation. We adopt the \citet{zhang2024adacs} protocol: a random 300/100/100 patient-level split (seed 42) yielding 600/200/200 train/val/test 2-D image pairs after including both 2CH and 4CH views per patient. Images are resampled to 1.0~mm in-plane and center-cropped to \(128 \times 128\).

\paragraph{Evaluation.}
ED\(\rightarrow\)ES Dice on LV myocardium, reported at seed 42, with cross-seed stability across seeds 42, 43, 44 characterised below. \citet{zhang2024adacs} does not report standard deviations for the Initial and Elastix baselines.

\paragraph{Cross-seed stability (Table~\ref{tab:camus}).}
The Table~\ref{tab:camus} number is reported at seed 42. Across three independent training seeds (42, 43, 44) under the recipe above, the LV-myocardium Dice values are 0.8204, 0.8190, and 0.8180 (cross-seed mean 0.8191, cross-seed standard deviation 0.001). The reported margin over VoxelMorph + AdaCS \citep{zhang2024adacs} is therefore within cross-seed variability, consistent with the parity claim in Section~\ref{sec:camus}.

\subsection{Quantum Harmonic Oscillator}
\label{app:qho}

This appendix supplements \S\ref{sec:qho} with the full architectural and training specification for the QHO experiment, a discussion of failure modes at large \(\tau\), and positioning relative to existing literature.

\paragraph{Common setup.}
Position grid \(x \in [-8, 8]\) with \(N = 128\) points; \(\hbar = \omega = m = 1\); the first \(n_{\max}=16\) Hermite-function eigenstates computed analytically via \texttt{scipy.special.hermite} and discretely L2-renormalized. Each training pair samples a random complex unit-norm coefficient vector \(c\in\mathbb{C}^{16}\), forms \(\Psi(0) = \sum_{n} c_n \psi_n^{\mathrm{true}}\), and time-evolves analytically eigenstate-by-eigenstate to obtain \(\Psi(\tau) = \sum_n c_n e^{-iE_n\tau/\hbar}\psi_n^{\mathrm{true}}\).

\begin{figure}[!ht]
\centering
\includegraphics[width=\linewidth]{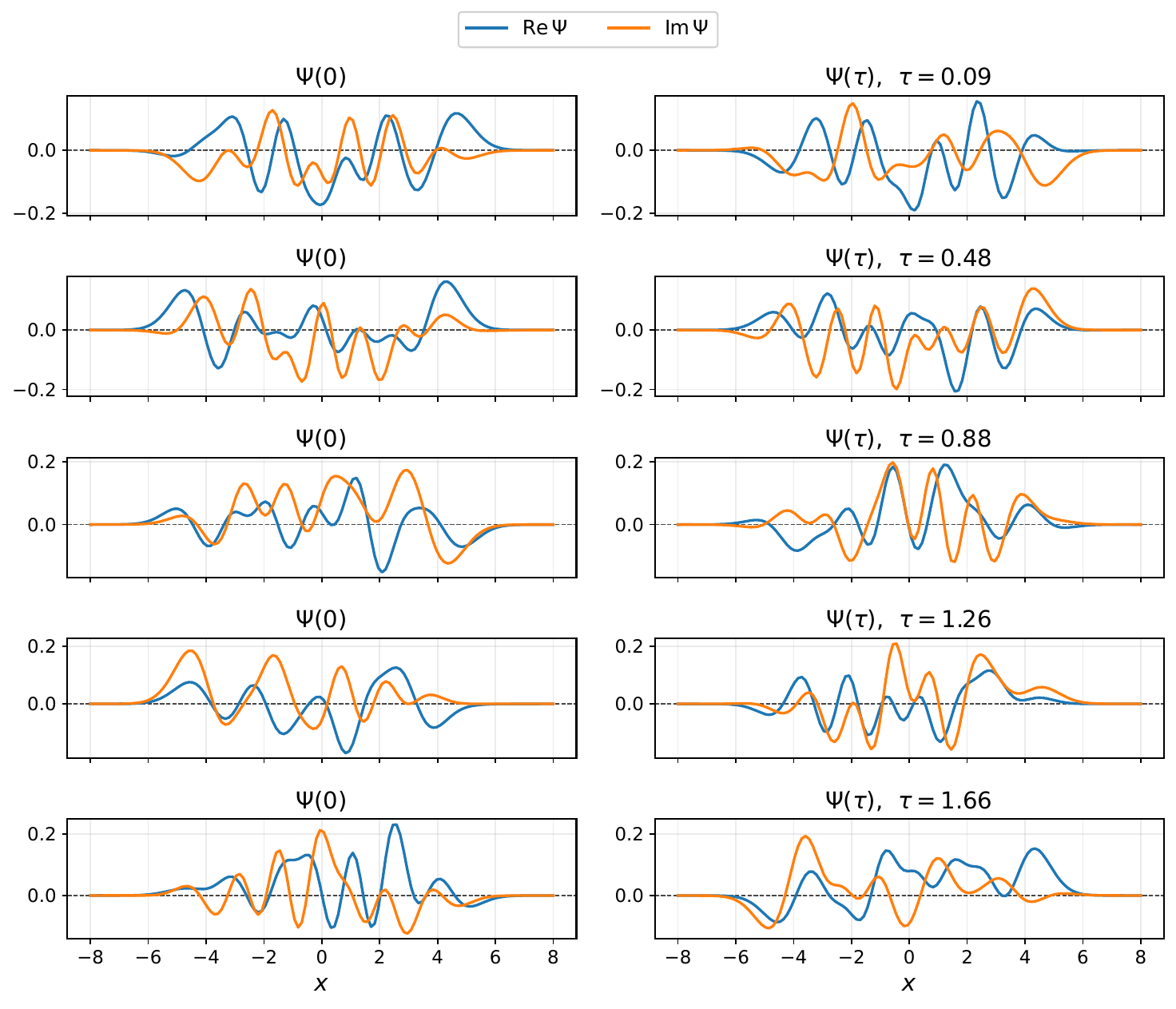}
\caption{Example training pairs \(\Psi(0)\) and \(\Psi(\tau)\): random complex unit-norm superpositions of the first $n_{\max}=16$ Hermite eigenstates, time-evolved analytically by a $\tau$ drawn uniformly from $[\tau_{\min}, \tau_{\max}]$. The model receives only the pair, not $\tau$ or the coefficients.}
\label{fig:qho_psi_pairs}
\end{figure}

\paragraph{Architecture and training.}
Sampler: $\tau_b$ drawn uniformly from $[\tau_{\min}, \tau_{\max}]$ per pair, not provided to the model; we report results at $\tau_{\max} \in \{0.6, 1.7\}$ (\S\ref{sec:qho}). Architecture: bank $\boldsymbol{\psi}_k \in \mathbb{R}^N$ ($K=16$, unit-normalized at use time); learned phase angles \(\delta_k\); per-filter learned rotation rate $E_k$ ($K$-vector); permutation-equivariant readout MLP $\mathcal{M}$ that ingests $\{(E_k, \cos\theta_k(b), \sin\theta_k(b), |z_k(b)|)\}_{k=1}^K$ per pair, pools across the filter axis with mean$+$max, and outputs $\hat\tau_b \in [\tau_{\min}, \tau_{\max}]$ via sigmoid range-scaling (hidden width 32, $\approx 3.3$k parameters). Reconstruction is closed-form in the diagonalized basis,
\[
\hat c_k(b) = c_k(\Psi_0^{(b)}) \cdot \exp\!\bigl(-i\,E_k\,\hat\tau_b\bigr),
\qquad
\widehat{\Psi}(\tau_b) = \sum_k \hat c_k(b)\,\boldsymbol{\psi}_k.
\]
Total parameters: $KN + K + |\mathcal{M}| = 2{,}048 + 16 + 3{,}329 \approx 5.4$k. Loss: magnitude-NCC reconstruction against $\Psi(\tau_b)$ + residual ($\mathcal{L}_{\mathrm{res}}$, Eq.~\eqref{eq:residual}) + diversity ($\mathcal{L}_{\mathrm{div}}$, off-diagonal Gram squared) + span coverage ($\mathcal{L}_{\mathrm{span}} = \mathbb{E}[\|\Psi\|^2 - \sum_k |c_k|^2]$); all weights $1.0$. Adam at $10^{-3}$, batch 64, 8{,}000 steps. Energies are recovered up to a single global gauge offset (magnitude-NCC's invariance to per-pair global phase), closed by a one-scalar variational anchor
\[
E_0^{\mathrm{calibrated}} = \langle\boldsymbol{\psi}_0^{\mathrm{learned}}|\widehat{H}|\boldsymbol{\psi}_0^{\mathrm{learned}}\rangle,
\qquad
\langle\psi|\widehat{H}|\psi\rangle = -i\,\partial_t\bigl\langle\psi\bigl|U(t)\bigr|\psi\bigr\rangle\bigr|_{t=0},
\]
read from the single-frequency oscillation of $\langle\psi_0^{\mathrm{learned}}|U(t)|\psi_0^{\mathrm{learned}}\rangle$ across pair times.

\begin{figure}[!ht]
\centering
\includegraphics[width=\linewidth]{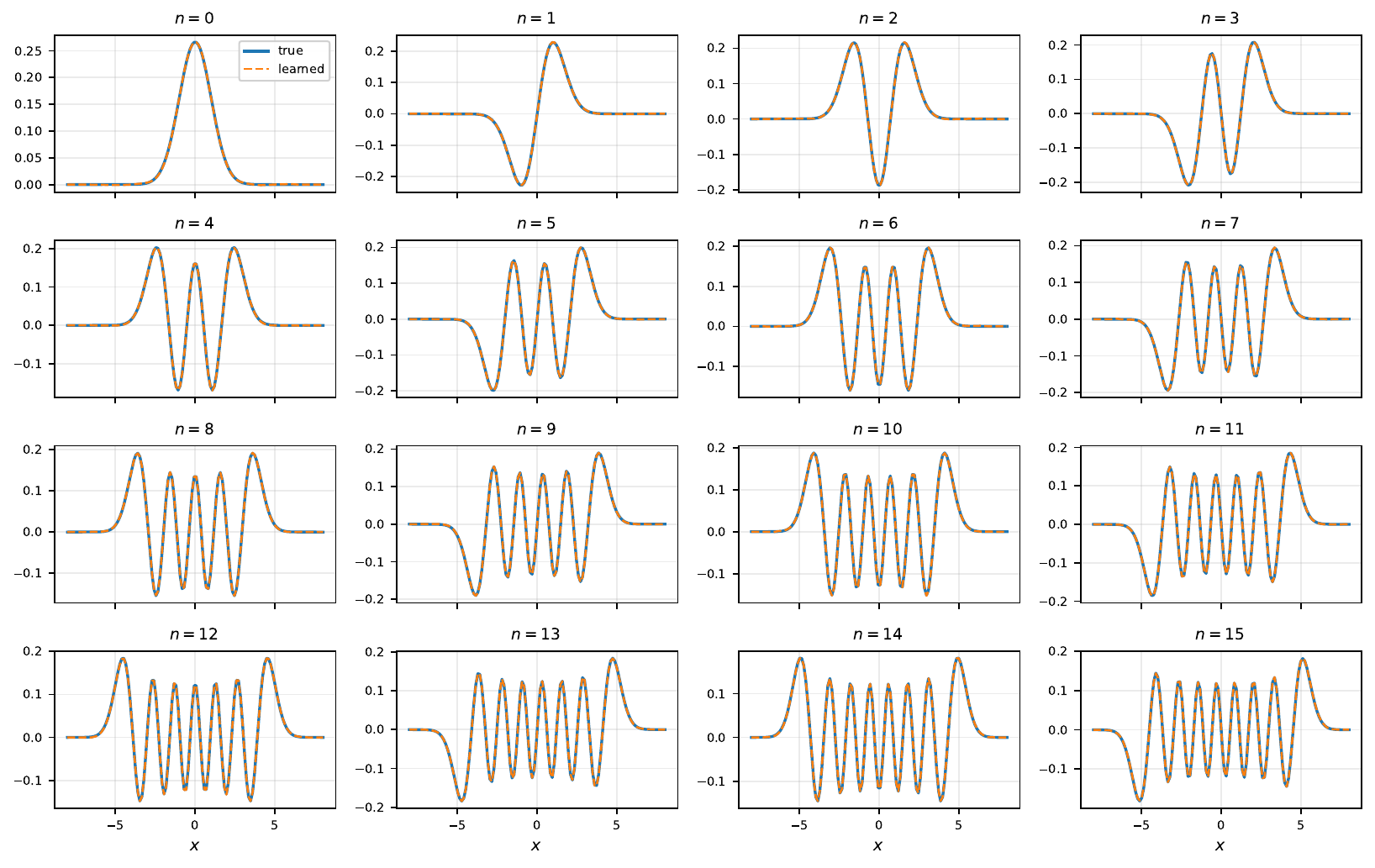}
\caption{Learned filters (orange dashed) overlaid on the analytic Hermite-function eigenstates (blue) for $n = 0, \ldots, 15$, sign-aligned. All 16 filters coincide with the analytic eigenstates (median overlap $1.000$).}
\label{fig:qho_eigenstates_taublind}
\end{figure}

\paragraph{Failure modes at large $\tau$.}
As $\tau_{\max}$ grows past $2\pi/E_{\max}$, the per-pair phase $E_n \tau_b/\hbar$ exceeds the principal branch and the reconstruction loss accepts any wrap count $E_n \tau_b/\hbar + 2\pi m$ as a valid solution. The per-filter rate $E_k$ provides per-filter robustness (each filter settles on a single consistent unwrap), but different filters in the bank can settle on different branches without mutual penalty, producing the spectrum fragmentation in Figure~\ref{fig:qho_eigenstates} (\S\ref{sec:qho}). Eigenstate recovery is unaffected: the structural primitive that identifies invariant subspaces (the residual) does not depend on $\tau$.

\paragraph{Position vs. existing literature.}
Neural Phase Correlation (NPC) recovers the eigenstructure of a unitary operator directly from observation pairs: a single trained model returns the entire spectrum, never sees the Hamiltonian, and does not require the per-pair time interval $\tau$ to be observed or binned. Koopman-net learning \citep{lusch2018deep, takeishi2017koopman} instead approximates the operator on a learned latent rather than recovering the eigenstructure directly. Neural-quantum-state variational Monte Carlo \citep{carleo2017solving} requires the Hamiltonian and recovers one eigenstate per training run. Physics-informed eigenvalue learning \citep{jin2022pinn} and Hamiltonian learning via neural ODEs \citep{kao2024hamiltonianlearning} both touch $H$ explicitly during training. Classical dynamic mode decomposition \citep{schmid2010dmd} recovers the spectrum from snapshot pairs at a single fixed $\tau$; the bilinear-interaction residual's $\tau$-invariance is what lets NPC operate without per-pair $\tau$ alignment, where DMD-style methods would require per-$\tau$ binning.

\section{ACDC ablations}
\label{sec:scaling}
\label{app:scaling}

\subsection{Residual-masking sweep}
\label{app:scaling-ksweep}

\input{figures/figure_ksweep_metrics.tex}

We ablate the residual-masking mechanism by varying the top-\(K\) mask budget at fixed filter-bank capacity \(C = 128\), training each variant under our preferred recipe from Appendix~\ref{app:acdc-protocol} (NCC similarity, diffusion regularization, Jacobian-determinant fold penalty; seed 43, batch size 100 for \(C = 128\) and 50 for \(C \geq 256\) due to memory, otherwise identical to the Table~\ref{tab:acdc_ncc} protocol). The sweep covers \(K \in \{1, 8, 32, 64, 96, 128\}\); \(K = 128\) recovers the unmasked baseline in which every pair contributes unconditionally. Figure~\ref{fig:ksweep_metrics} reports Dice as a function of \(K\) (left) and the corresponding learned filter banks at each \(K\) (right).

\paragraph{Dice peaks}
Pooled Dice rises from \(K = 1\) (\(0.831\)) through \(K = 64\) (\(0.857\)) and then declines as \(K\) approaches \(C\), ending at \(0.841\) unmasked. The drop from peak to unmasked is \(0.016\), more than \(5\times\) the cross-seed standard deviation reported in Appendix~\ref{app:acdc-protocol}. Both directions follow the same shape: ED\(\rightarrow\)ES gains \(0.024\) from unmasked to peak (\(0.819 \to 0.834\)), ES\(\rightarrow\)ED gains \(0.019\) (\(0.862 \to 0.881\)). Pulling the full bank into the gradient does not just stop helping, it actively hurts, because the high-residual pairs contribute mis-aligned phase signal at inference.


\paragraph{\(K = 1\) is the most informative point.}
A single residual-selected pair already delivers \(0.831\) Dice, within \(0.010\) of the unmasked \(C = 128\) baseline. Even one specialist captures most of the registration signal, providing direct evidence that the residual is selecting the right pair (the one whose span is most invariant under the local transformation), not merely providing a generic sparsity bottleneck. This connects to the necessity-and-sufficiency result in Appendix~\ref{app:residual}: \(r_k^2 = 0\) iff \(\mathrm{span}(\boldsymbol{\psi}_k)\) is invariant under \(L_p\), so masking by residual selects pairs with the theoretical guarantee, not a heuristic proxy.

\paragraph{Bank-level diagnostic.}
The shape transition in Figure~\ref{fig:ksweep_metrics} (right) mirrors the per-pair residual distribution diagnostic of Figure~\ref{fig:residual_dist} (main text). The \(K\)-sweep makes the supervision-incentive picture from Section~\ref{sec:method} explicit: as \(K\) shrinks, the supervision concentrates on a tighter specialist set, producing the cleaner Gabor structure visible in Figure~\ref{fig:ksweep_metrics} (right). Visual inspection reveals an increase in the noise and granularity of the filter pairs (specialists \(\to\) generalists) for increasing \(K\), except for the apparent lucidity at \(K = 64\) which coincides with the sharp Dice peak.

\paragraph{Capacity scaling.}
Doubling the bank to \(C = 256\) leaves the optimal mask budget within the same regime, peaking at \(K = 192\) with Pooled Dice \(0.865\), only \(0.008\) above the \(C = 128\) peak (Figure~\ref{fig:ksweep_metrics}, purple diamonds). Quadrupling to \(C = 512\) at \(K = 256\) returns \(0.862\), no better than the unmasked \(C = 256\) limit. The peak-then-decay shape persists at \(C = 256\), with the \(K = 256\) endpoint dipping back to the unmasked level (\(0.859\)). Capacity beyond \(C = 128\) is therefore not the binding constraint at this benchmark; specialization through a residual-selected subset is.

\subsection{ODE step count}
\label{app:scaling-ode}

\input{figures/figure_ode_steps.tex}

We sweep the number of integration steps used by the displacement-field ODE while holding all other hyperparameters fixed at the \(C = 128\), \(K = 64\) baseline of Appendix~\ref{app:acdc-protocol}.

\paragraph{Plateau onset.}
Pooled Dice rises sharply from a single Euler step (\(0.840\)) to ODE\,4 (\(0.862\)) and then plateaus through ODE\,13 (\(0.859\)): the spread across the plateau (\(0.005\)) sits within the cross-seed standard deviation reported in Appendix~\ref{app:acdc-protocol}. The default of ODE\,10 used throughout the paper is inside this plateau, not at its edge.

\paragraph{Inference-cost implication.}
The plateau means that any step count between \(4\) and \(13\) is statistically equivalent on Pooled Dice. ODE\,4 attains the same accuracy as the default at roughly \(0.4\times\) the per-pair inference cost (Appendix~\ref{app:scaling-compute}); the choice of ODE\,10 in the released checkpoints reflects training-time conservatism rather than an accuracy requirement.

\subsection{Compute profile}
\label{app:scaling-compute}

\paragraph{Setup.}
We profile the default \(C = 128\), \(K = 64\), ODE\,10 model on a single NVIDIA A100-SXM4-40GB at batch size 1 and input resolution \(128 \times 128\), using \texttt{fvcore} for FLOPs and 100 forward passes (20 warmup) for wall-clock.

\paragraph{Parameters, FLOPs, and wall clock.}
The model has \(0.32\)M trainable parameters and \(3.29\) GFLOPs per forward pass. Mean per-pair inference time is \(75.7 \pm 1.7\) ms (median \(75.3\) ms, \(95\)th percentile \(77.8\) ms). Combined with the ODE-step plateau (Appendix~\ref{app:scaling-ode}), reducing to ODE\,4 brings the per-pair cost to approximately \(30\) ms with no measurable Dice cost; this is the operating point we recommend for inference-throughput-bound deployment.

\paragraph{Training compute.}
All training runs in the paper were performed on a single NVIDIA A100-SXM4-40GB GPU.

\section{Deformation breadth}
\label{app:breadth}

To verify that NPC can operate across a wide diversity of deformation regimes, we report two additional experiments beyond the cardiac MRI and echocardiography benchmarks of Section~\ref{sec:experiments}: a qualitative demonstration on 2D OASIS brain MRI (Section~\ref{app:oasis}) and a quantitative transfer to sub-meter synthetic-aperture-radar imagery (Section~\ref{app:sar}), the latter compared against the classical phase-correlation baseline.

\subsection{2D OASIS brain MRI}
\label{app:oasis}

OASIS pairs differ from ACDC in tissue contrast, anatomical structure, and scale of deformation: cortical folds and ventricles dominate the alignment task rather than ventricular contraction.

\paragraph{Setup.}
We train a single-stage \(C = 128\), \(K = 64\) model on 2D mid-axial slices from the OASIS-1 dataset \citep{marcus2007oasis} using NCC similarity, diffusion regularization, and ODE integration. This section is a qualitative demonstration: we do not compare against OASIS-specific baselines.

\paragraph{Qualitative registration.}
On four held-out test pairs (Figure~\ref{fig:oasis_combined}a), the deformation grid contracts smoothly across ODE steps to align cortical folds and ventricles between moving and fixed slices. The warp tightens monotonically with integration time and remains diffeomorphic at \(t = 15\) without visible folding.

\paragraph{Learned filter bank.}
The learned filter bank shows the same within-pair coherence as on ACDC, with tighter spatial localization (Figure~\ref{fig:oasis_combined}b).

\subsection{Sub-meter SAR}
\label{app:sar}

High-resolution SAR-SAR co-registration is a current bottleneck for downstream remote-sensing tasks (change detection, deformation monitoring, time-series analysis), where sub-pixel alignment is required. The field currently relies on hand-crafted phase-correlation methods; no deep-learning architecture exists, and no public benchmark is available\citep{li2025remote}. We therefore construct our own dataset and benchmark against the de facto standard (global phase correlation).

\paragraph{Setup.}
Our dataset consists of \(256 \times 256\) geo-coded SAR image pairs at \(0.5\)~m resolution from locations throughout Europe, Asia, Australia, North America, and South America. We present results on held-out test pairs taken over the Puerto Maldonado region of the Peruvian Amazon.

\paragraph{Registration results.}
Across the 40 sampled pairs (Figure~\ref{fig:sar_combined}a), the network beats classical phase correlation on all but one pair, which we show in the figure, demonstrating that the model transfers to phase correlation's operational domain without framework modifications.

\paragraph{Cascade filter bank.}
The coarse-stage filters (Figure~\ref{fig:sar_combined}b, top two rows) develop broader, speckle-tolerant patterns; the fine-stage filters (bottom two rows) tighten to higher spatial frequencies, consistent with the coarse-to-fine training schedule.

\input{figures/figure_oasis_combined.tex}
\input{figures/figure_sar_combined.tex}

\subsection{Filter localization across regimes}

The three regimes span a spectrum of filter localization. SAR filters are not spatially localized (Figure~\ref{fig:sar_combined}b). ACDC filters concentrate moderately, on the scale of anatomical features (Figure~\ref{fig:ksweep_metrics}). OASIS filters are the most localized, with effective half-periods of approximately 4 pixels, at the spatial scale of individual cortical features (Figure~\ref{fig:oasis_combined}b).

\section{Limitations}
\label{app:limitations}

The framework assumes the local transformation \(L_p\) is approximately orthogonal: locally, on small patches, for image registration; exactly for unitary dynamics. Some classes of orthogonal transformations are not captured by a trivial extension of the current architecture. In particular, optical flow with non-local orthogonal structure (transformations that mix pixels at distant locations) requires a non-local cost-volume formulation rather than the per-location bilinear interaction used here.

For the QHO experiment, we use only the first 16 Hermite eigenstates with a matched bank of 16 filter pairs; the architecture's behavior at higher truncation orders is not characterized. The reconstruction loss is invariant to per-pair global phase, which causes filters to fragment across unwrap branches at large \(\tau_{\max}\) (Appendix~\ref{app:qho}) and limits the model's ability to recover the eigenenergies (and therefore the Hamiltonian) at large time intervals.

The QHO experiment is also a purely algebraic exercise: training pairs are generated by analytic time-evolution of complex superpositions. The setup is not physically realizable in its current form because quantum-state measurements are destructive (wavefunction collapse), so a physical experiment cannot directly observe \(\Psi(0)\) and \(\Psi(\tau)\) as the model requires.


\end{document}

%% file: figures/registration_examples_acdc/figure_registration_examples_acdc.tex
%

\begin{figure}[!ht]
\centering
\setlength{\tabcolsep}{2pt}
\renewcommand{\arraystretch}{1.05}

\begin{tabular}{@{}ccccc@{}}
moving & fixed & warped & warped $-$ fixed & deformation \\
\includegraphics[width=0.190\linewidth]{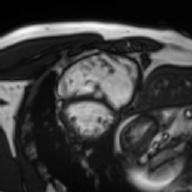} & \includegraphics[width=0.190\linewidth]{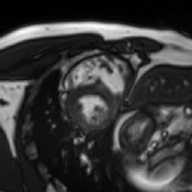} & \includegraphics[width=0.190\linewidth]{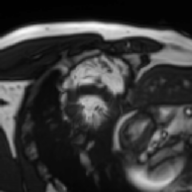} & \includegraphics[width=0.190\linewidth]{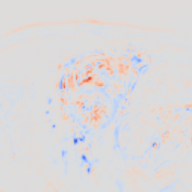} & \includegraphics[width=0.190\linewidth]{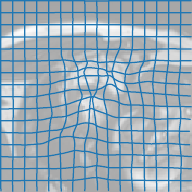} \\
\includegraphics[width=0.190\linewidth]{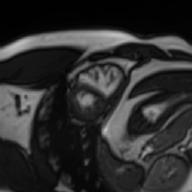} & \includegraphics[width=0.190\linewidth]{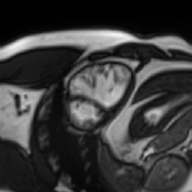} & \includegraphics[width=0.190\linewidth]{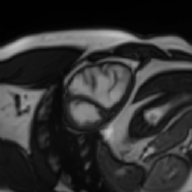} & \includegraphics[width=0.190\linewidth]{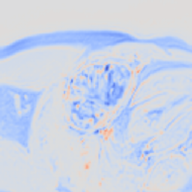} & \includegraphics[width=0.190\linewidth]{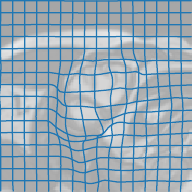} \\
\end{tabular}

\caption{Registration examples on ACDC test pairs. Columns: moving image; fixed image; warped moving image; signed difference (warped $-$ fixed) on a coolwarm scale; forward-displacement deformation field rendered as a deformed identity grid (the grid moves with the apparent image motion).}
\label{fig:registration_examples_acdc}
\end{figure}

%% file: figures/registration_examples_camus/figure_registration_examples_camus.tex
%

\begin{figure}[!ht]
\centering
\setlength{\tabcolsep}{2pt}
\renewcommand{\arraystretch}{1.05}

\begin{tabular}{@{}ccccc@{}}
moving & fixed & warped & warped $-$ fixed & deformation \\
\includegraphics[width=0.190\linewidth]{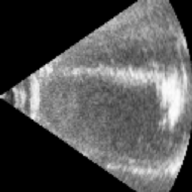} & \includegraphics[width=0.190\linewidth]{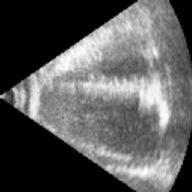} & \includegraphics[width=0.190\linewidth]{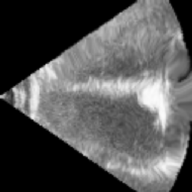} & \includegraphics[width=0.190\linewidth]{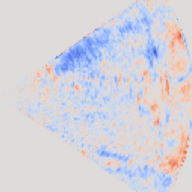} & \includegraphics[width=0.190\linewidth]{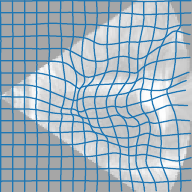} \\
\includegraphics[width=0.190\linewidth]{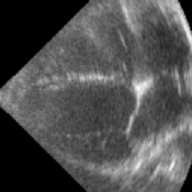} & \includegraphics[width=0.190\linewidth]{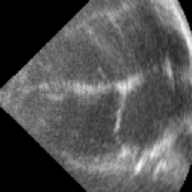} & \includegraphics[width=0.190\linewidth]{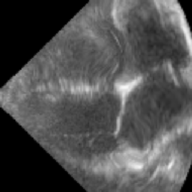} & \includegraphics[width=0.190\linewidth]{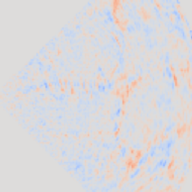} & \includegraphics[width=0.190\linewidth]{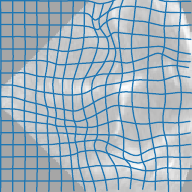} \\
\end{tabular}

\caption{Registration examples on CAMUS test pairs. Columns: moving image; fixed image; warped moving image; signed difference (warped $-$ fixed) on a coolwarm scale; forward-displacement deformation field rendered as a deformed identity grid (the grid moves with the apparent image motion).}
\label{fig:registration_examples_camus}
\end{figure}

%% file: figures/figure_ksweep_metrics.tex
%

\begin{figure}[!ht]
\centering
\begin{minipage}[c]{0.40\linewidth}
\centering
\includegraphics[width=\linewidth]{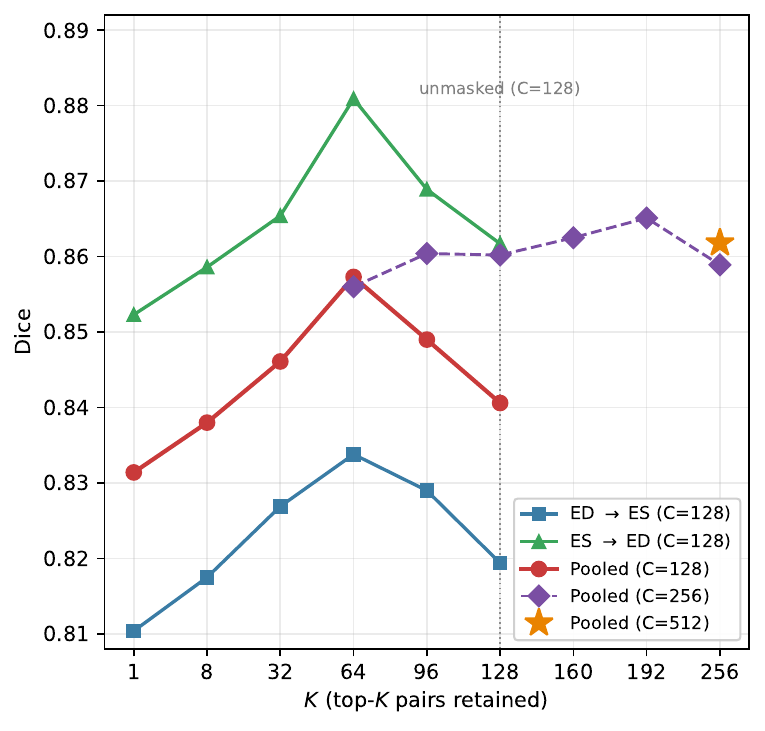}
\end{minipage}\hfill
\begin{minipage}[c]{0.58\linewidth}
\centering
\setlength{\tabcolsep}{1.5pt}
\renewcommand{\arraystretch}{1.05}
\begin{tabular}{@{}ccc@{}}
\subcaptionbox{$K = 1$\label{fig:ksweep_filters_k1}}{%
  \includegraphics[width=0.31\linewidth]{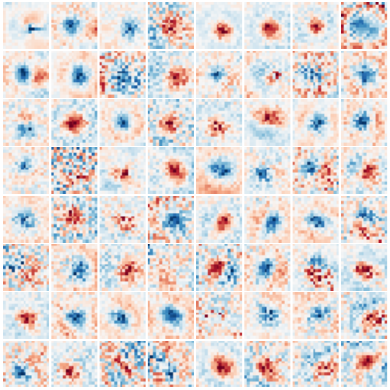}} &
\subcaptionbox{$K = 8$\label{fig:ksweep_filters_k8}}{%
  \includegraphics[width=0.31\linewidth]{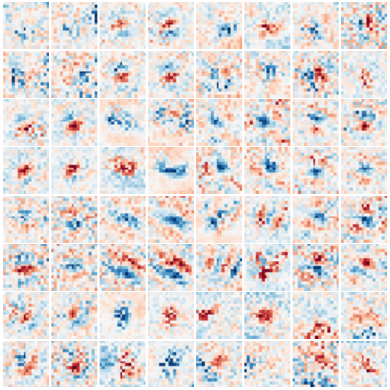}} &
\subcaptionbox{$K = 32$\label{fig:ksweep_filters_k32}}{%
  \includegraphics[width=0.31\linewidth]{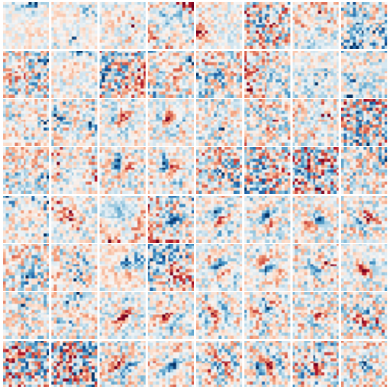}} \\[1pt]
\subcaptionbox{$K = 64$\label{fig:ksweep_filters_k64}}{%
  \includegraphics[width=0.31\linewidth]{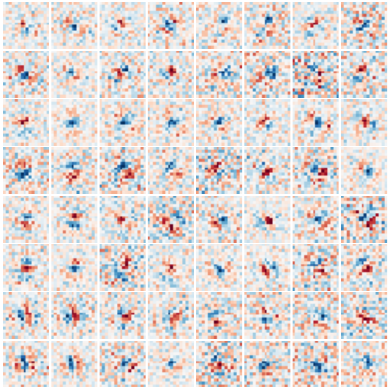}} &
\subcaptionbox{$K = 96$\label{fig:ksweep_filters_k96}}{%
  \includegraphics[width=0.31\linewidth]{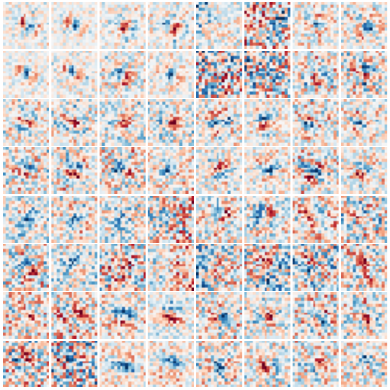}} &
\subcaptionbox{$K = 128$\label{fig:ksweep_filters_k128}}{%
  \includegraphics[width=0.31\linewidth]{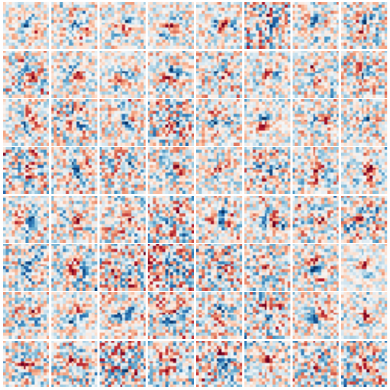}} \\
\end{tabular}
\end{minipage}
\caption{K-sweep ablation on ACDC. The Pooled (\(C = 128\)) curve forms a smooth dome with peak at \(K = 64\) (\(0.857\)), \(0.016\) above the unmasked limit at \(K = C = 128\). Purple diamonds overlay Pooled Dice at \(C = 256\) (bs=50) for \(K \in \{64, 96, 128, 160, 192, 256\}\); the highest Pooled Dice in the sweep is \(C = 256\), \(K = 192\) (\(0.865\)). The orange star marks \(C = 512\), \(K = 256\) (\(0.862\)). \textbf{Right:} 16 filter pair quads of each \(C = 128\) bank shown as 2\(\times\)2 quads of \((\boldsymbol{\psi}_{2k-1}, \boldsymbol{\phi}_{2k-1}, \boldsymbol{\psi}_{2k}, \boldsymbol{\phi}_{2k})\). Small \(K\) develops clean Gabor-like specialists with strong within-pair coherence; \(K\) approaching the unmasked limit produces noise-like generalists with no coherent within-pair structure.}
\label{fig:ksweep_metrics}
\end{figure}

%% file: figures/figure_ode_steps.tex
%

\begin{figure}[!ht]
\centering
\includegraphics[width=0.55\linewidth]{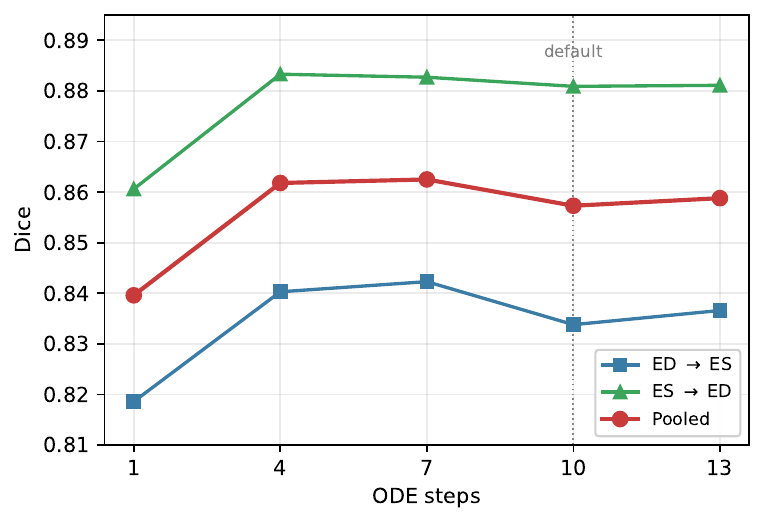}
\caption{ODE-step ablation on ACDC for the \(C{=}128\), \(K{=}64\) baseline. ODE\,10 is the default used throughout the paper (dotted vertical). Pooled Dice plateaus from ODE\,4 onward, and the spread between ODE\,4 and ODE\,13 (\(0.857\)\,--\,\(0.863\)) is within the cross-seed standard deviation reported in Appendix~\ref{app:acdc-protocol}.}
\label{fig:ode_steps}
\end{figure}

%% file: figures/figure_oasis_combined.tex
%

\begin{figure}[!ht]
\centering

\begin{subfigure}[t]{\linewidth}
\centering
\setlength{\tabcolsep}{1pt}
\renewcommand{\arraystretch}{0.65}
\begin{tabular}{@{}ccccc@{}}
\scriptsize Moving & \scriptsize $t=1$ & \scriptsize $t=8$ & \scriptsize $t=15$ & \scriptsize Fixed \\[0pt]
\includegraphics[width=0.135\linewidth]{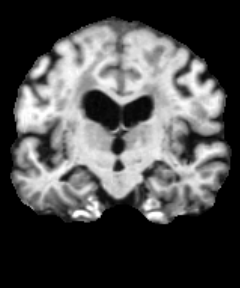} &
\includegraphics[width=0.135\linewidth]{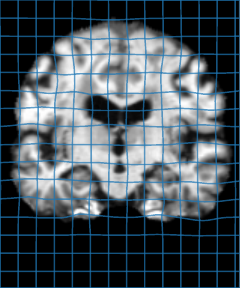} &
\includegraphics[width=0.135\linewidth]{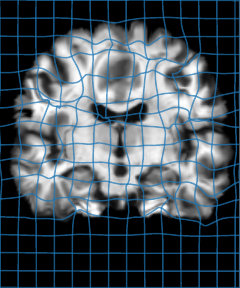} &
\includegraphics[width=0.135\linewidth]{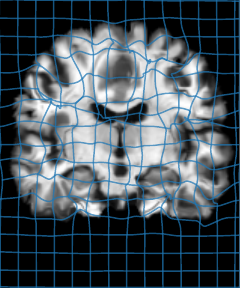} &
\includegraphics[width=0.135\linewidth]{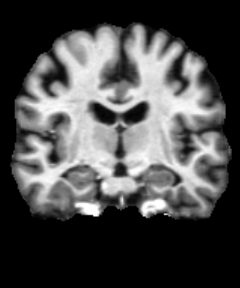} \\
\includegraphics[width=0.135\linewidth]{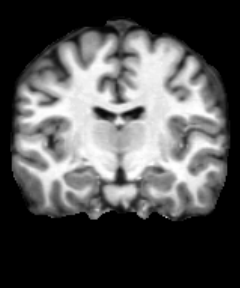} &
\includegraphics[width=0.135\linewidth]{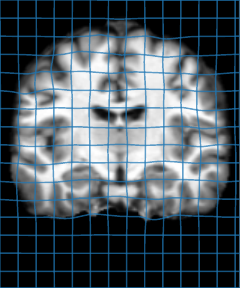} &
\includegraphics[width=0.135\linewidth]{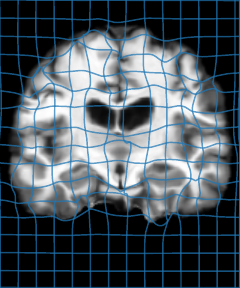} &
\includegraphics[width=0.135\linewidth]{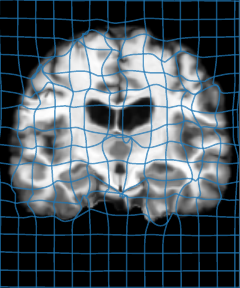} &
\includegraphics[width=0.135\linewidth]{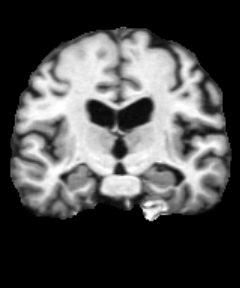} \\
\includegraphics[width=0.135\linewidth]{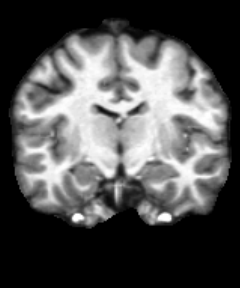} &
\includegraphics[width=0.135\linewidth]{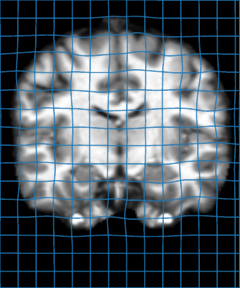} &
\includegraphics[width=0.135\linewidth]{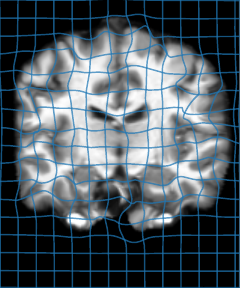} &
\includegraphics[width=0.135\linewidth]{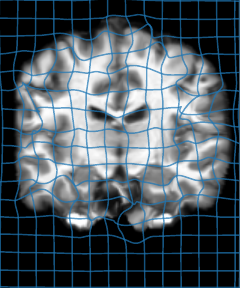} &
\includegraphics[width=0.135\linewidth]{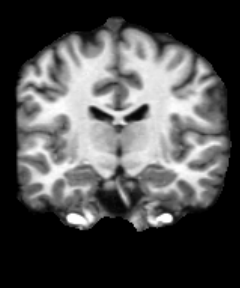} \\
\includegraphics[width=0.135\linewidth]{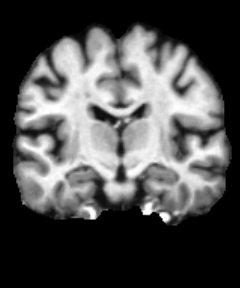} &
\includegraphics[width=0.135\linewidth]{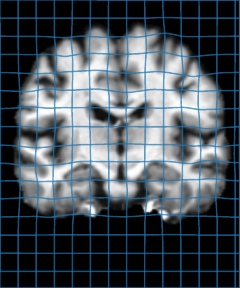} &
\includegraphics[width=0.135\linewidth]{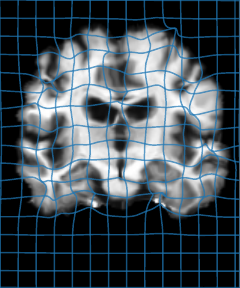} &
\includegraphics[width=0.135\linewidth]{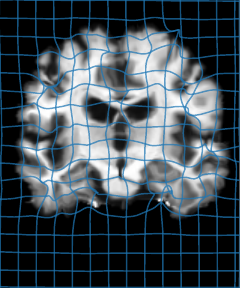} &
\includegraphics[width=0.135\linewidth]{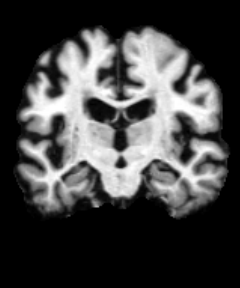} \\
\end{tabular}
\caption{Registration timelapse on four held-out test pairs. Columns: moving, intermediate warps at ODE steps $t \in \{1, 8, 15\}$ with deformation grid overlay, and fixed.}
\label{fig:oasis_timelapse}
\end{subfigure}

\vspace{4pt}

\begin{subfigure}[t]{\linewidth}
\centering
\setlength{\tabcolsep}{1pt}
\renewcommand{\arraystretch}{0.6}
\begin{tabular}{@{}cccc@{}}
\includegraphics[width=0.135\linewidth]{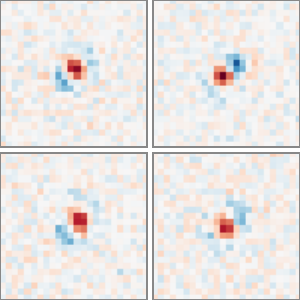} &
\includegraphics[width=0.135\linewidth]{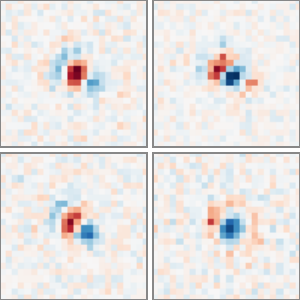} &
\includegraphics[width=0.135\linewidth]{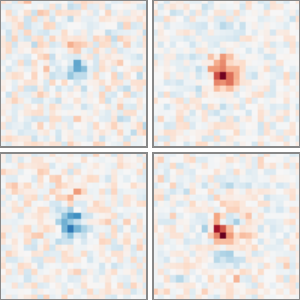} &
\includegraphics[width=0.135\linewidth]{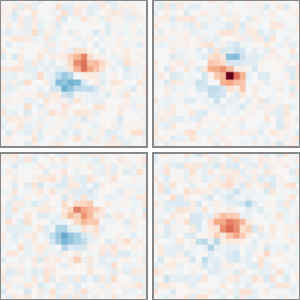} \\
\includegraphics[width=0.135\linewidth]{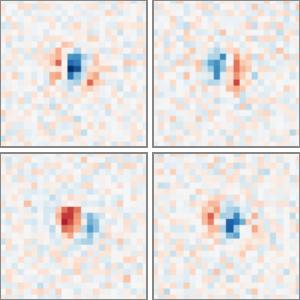} &
\includegraphics[width=0.135\linewidth]{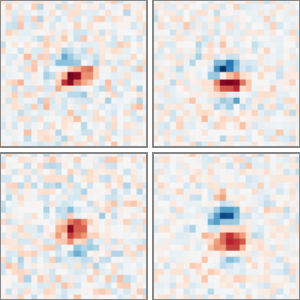} &
\includegraphics[width=0.135\linewidth]{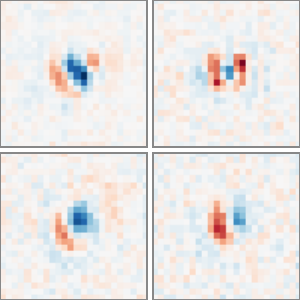} &
\includegraphics[width=0.135\linewidth]{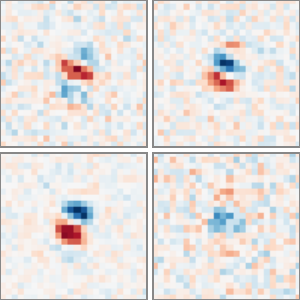} \\
\includegraphics[width=0.135\linewidth]{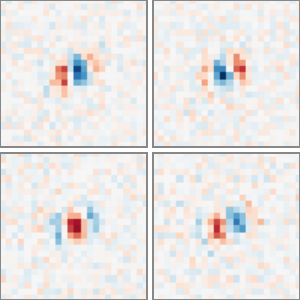} &
\includegraphics[width=0.135\linewidth]{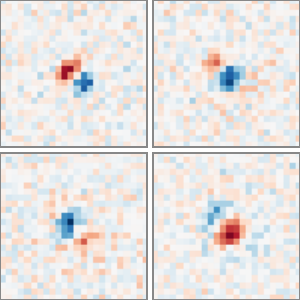} &
\includegraphics[width=0.135\linewidth]{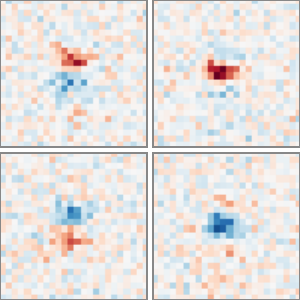} &
\includegraphics[width=0.135\linewidth]{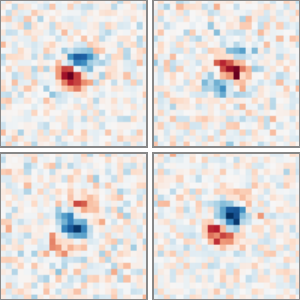} \\
\includegraphics[width=0.135\linewidth]{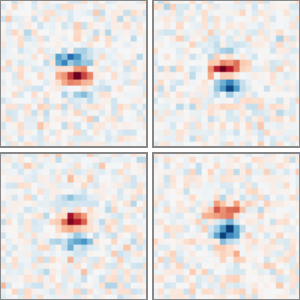} &
\includegraphics[width=0.135\linewidth]{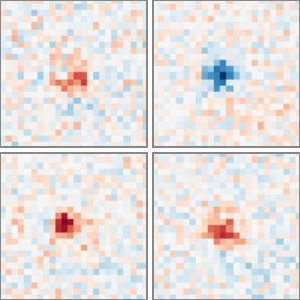} &
\includegraphics[width=0.135\linewidth]{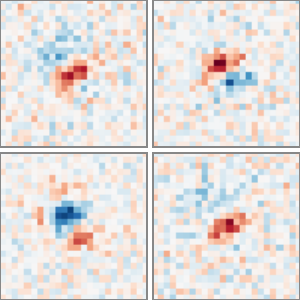} &
\includegraphics[width=0.135\linewidth]{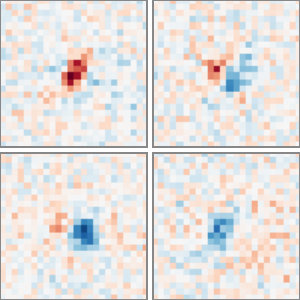} \\
\end{tabular}
\caption{Sixteen learned quad filter pairs from the OASIS bank. Each panel is a $2 \times 2$ quad of $(\boldsymbol{\psi}_{2k-1}, \boldsymbol{\phi}_{2k-1}, \boldsymbol{\psi}_{2k}, \boldsymbol{\phi}_{2k})$.}
\label{fig:oasis_quad_panels}
\end{subfigure}

\caption{OASIS 2D cross-domain demonstration. \textbf{(a)} Deformation grid contracts smoothly across ODE steps to align cortical folds and ventricles, remaining diffeomorphic at $t = 15$ without visible folding. \textbf{(b)} Learned filter bank exhibits the same spatially localized, within-pair coherent structure observed on ACDC (Figure~\ref{fig:ksweep_metrics}, right).}
\label{fig:oasis_combined}
\end{figure}

%% file: figures/figure_sar_combined.tex
%

\newcommand{\sarpanel}[1]{%
  \includegraphics[width=0.155\linewidth]{figures/sar_horizon/#1.pdf}}
\newcommand{\sarncc}[1]{\scriptsize\textsc{ncc}\,$=\,#1$}
\newcommand{\sarquad}[1]{%
  \includegraphics[width=0.135\linewidth]{figures/c2f_quad_panels/#1.pdf}}

\begin{figure}[!ht]
\centering

\begin{subfigure}[t]{\linewidth}
\centering
\setlength{\tabcolsep}{1.5pt}
\renewcommand{\arraystretch}{0.85}
\begin{tabular}{ccccc}
\scriptsize\textbf{Moving} & \scriptsize\textbf{Fixed} & \scriptsize\textbf{Before} & \scriptsize\textbf{Ours} & \scriptsize\textbf{Phase corr.} \\
\addlinespace[1pt]
   &   & \sarncc{0.002} & \sarncc{0.543} & \sarncc{0.482} \\
\sarpanel{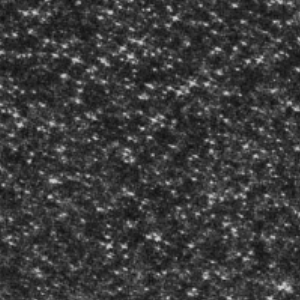} &
\sarpanel{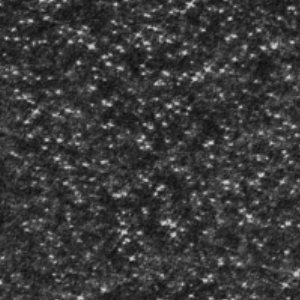} &
\sarpanel{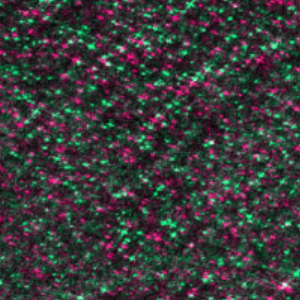} &
\sarpanel{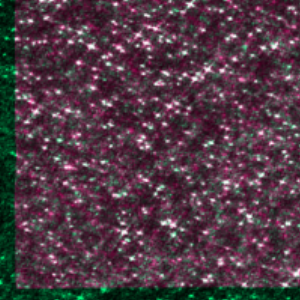} &
\sarpanel{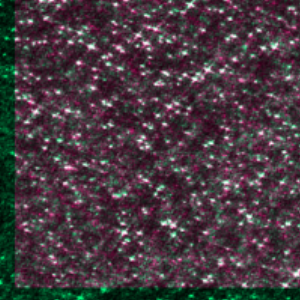} \\
\addlinespace[2pt]
   &   & \sarncc{0.075} & \sarncc{0.492} & \sarncc{0.408} \\
\sarpanel{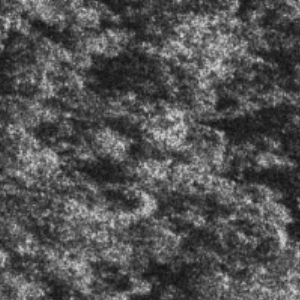} &
\sarpanel{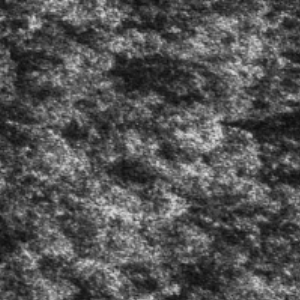} &
\sarpanel{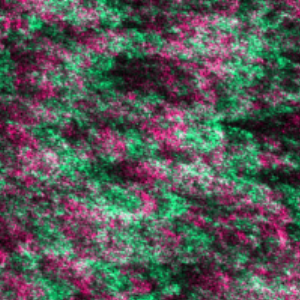} &
\sarpanel{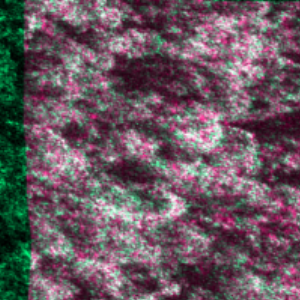} &
\sarpanel{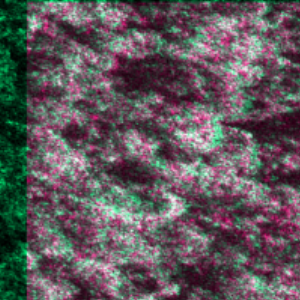} \\
\addlinespace[2pt]
   &   & \sarncc{0.505} & \sarncc{0.693} & \sarncc{0.704} \\
\sarpanel{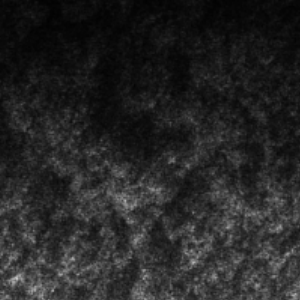} &
\sarpanel{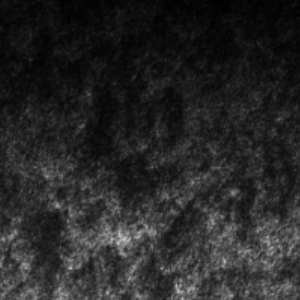} &
\sarpanel{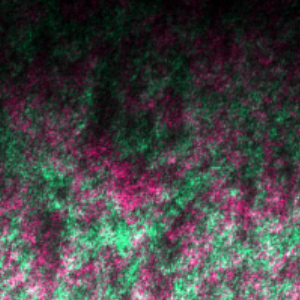} &
\sarpanel{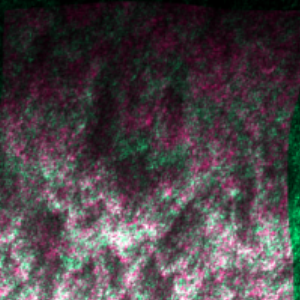} &
\sarpanel{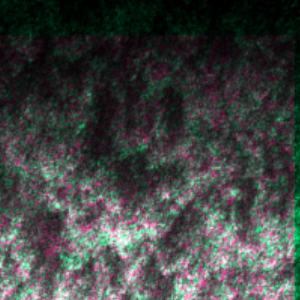} \\
\addlinespace[2pt]
   &   & \sarncc{0.430} & \sarncc{0.570} & \sarncc{0.473} \\
\sarpanel{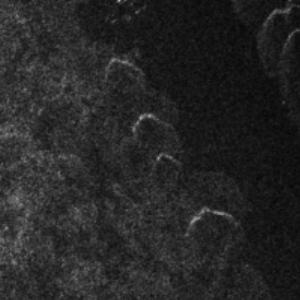} &
\sarpanel{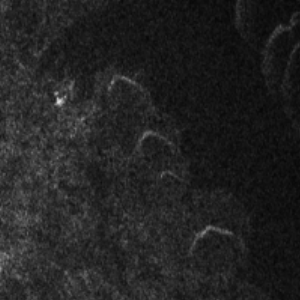} &
\sarpanel{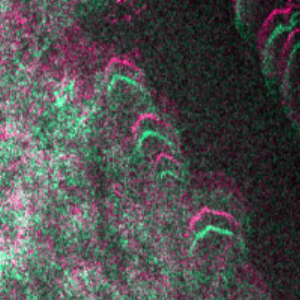} &
\sarpanel{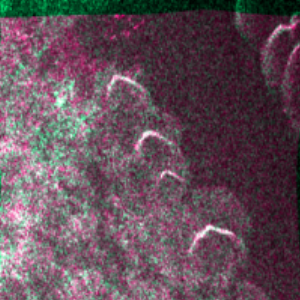} &
\sarpanel{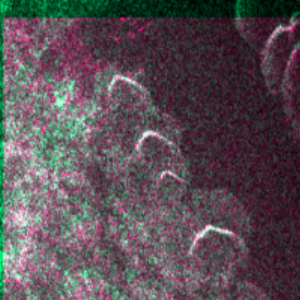} \\
\end{tabular}
\caption{SAR-to-SAR registration on Puerto Maldonado Umbra acquisitions (Feb24 $\to$ Mar20). Network NCC exceeds phase-correlation NCC on all but one sample (row 3, included to show the failure case).}
\label{fig:sar_horizon}
\end{subfigure}

\vspace{4pt}

\begin{subfigure}[t]{\linewidth}
\centering
\setlength{\tabcolsep}{1pt}
\renewcommand{\arraystretch}{0.6}
\begin{tabular}{@{}cccc@{}}
\sarquad{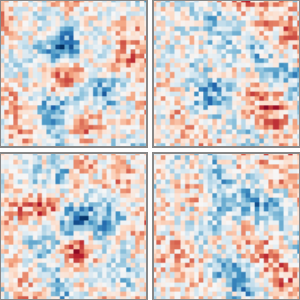} &
\sarquad{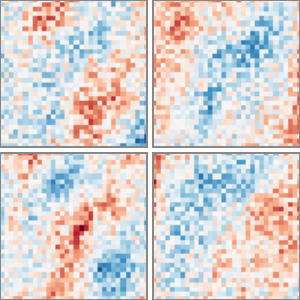} &
\sarquad{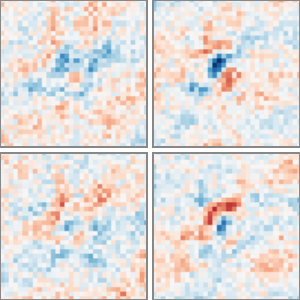} &
\sarquad{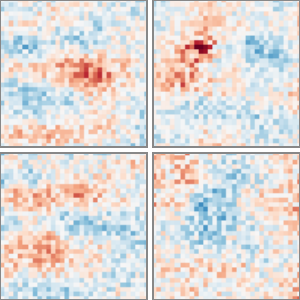} \\
\sarquad{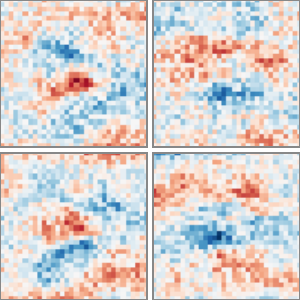} &
\sarquad{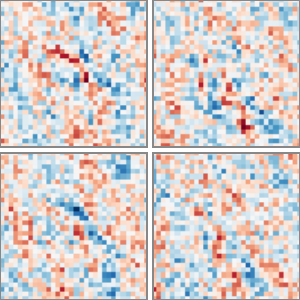} &
\sarquad{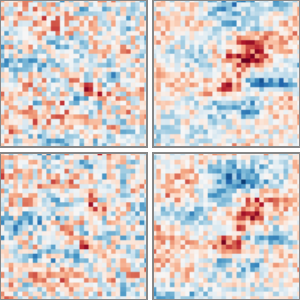} &
\sarquad{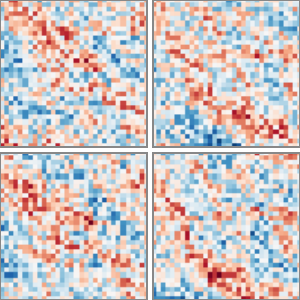} \\
\sarquad{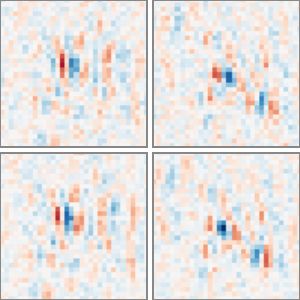} &
\sarquad{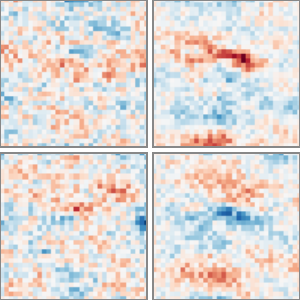} &
\sarquad{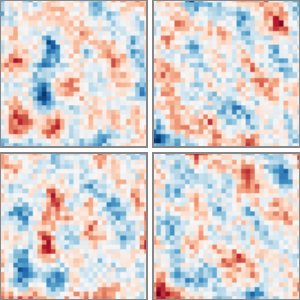} &
\sarquad{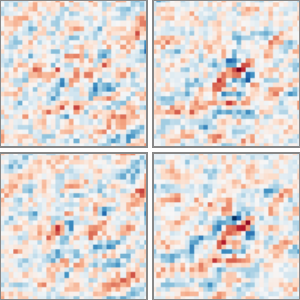} \\
\sarquad{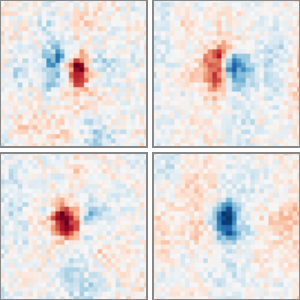} &
\sarquad{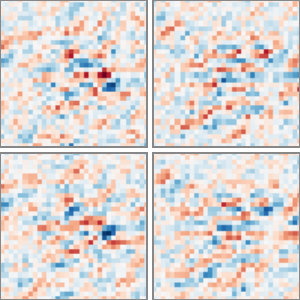} &
\sarquad{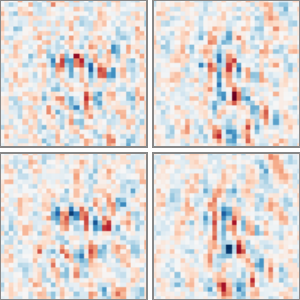} &
\sarquad{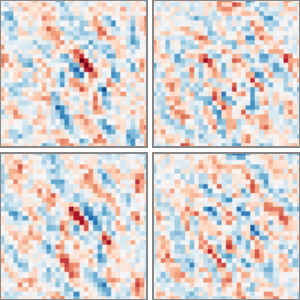} \\
\end{tabular}
\caption{Sixteen learned quad filter pairs from the cascade: coarse stage (top two rows) and fine stage (bottom two rows). Each panel is a $2 \times 2$ quad of $(\boldsymbol{\psi}_{2k-1}, \boldsymbol{\phi}_{2k-1}, \boldsymbol{\psi}_{2k}, \boldsymbol{\phi}_{2k})$.}
\label{fig:sar_quad_panels}
\end{subfigure}

\caption{SAR cross-domain transfer. \textbf{(a)} Registration results on four held-out Umbra image pairs. \textbf{(b)} Coarse- and fine-stage learned filter banks. The coarse stage develops broader, speckle-tolerant patterns; the fine stage tightens to higher spatial frequencies.}
\label{fig:sar_combined}
\end{figure}